\RequirePackage{fix-cm}
\documentclass{article} % For LaTeX2e
\usepackage[OT1]{fontenc}
\usepackage{preprint,times}

\usepackage{amsmath,amsfonts,bm}

\def\eqref#1{equation~\ref{#1}}
\def\1{\bm{1}}

\DeclareMathAlphabet{\mathsfit}{\encodingdefault}{\sfdefault}{m}{sl}
\SetMathAlphabet{\mathsfit}{bold}{\encodingdefault}{\sfdefault}{bx}{n}

\usepackage{xcolor}
\usepackage{graphicx}
\usepackage{wrapfig}
\usepackage{amssymb}
\usepackage{booktabs}
\usepackage{multirow}
\usepackage{longtable}
\usepackage{colortbl}
\usepackage{amsthm}
\usepackage{amsmath}
 
\usepackage{calc}
\usepackage{caption}      % \captionof 需要
\usepackage{needspace}
\usepackage{placeins}
\usepackage{float}
\definecolor{citeblue}{HTML}{16A6C9}

\newtheoremstyle{quietprop}
  {3pt}{3pt}            % 上方、下方间距
  {\normalfont}{}       % 正文字体、缩进
  {\itshape}{.}         % 标题字体、标题后标点
  { }{}                 % 标题后空格、标题格式
\theoremstyle{quietprop}       
\newtheorem*{proposition}{Proposition}

\usepackage[
    colorlinks=true,
    citecolor=citeblue,
    linkcolor=citeblue,
    urlcolor=citeblue
]{hyperref}

\newsavebox{\berbox}   % Figure 6 graphic, measured for the adjacent table

\usepackage{enumitem}
\usepackage{capt-of}   % \captionof: the table sharing Figure 6's float
\usepackage{makecell}  % two-line column headings in the narrow Table 2

\usepackage{mathtools}  % mathtools 提供 \coloneqq

\theoremstyle{plain}
\theoremstyle{remark}

\usepackage{cleveref}
\usepackage{url}
\usepackage[breakable]{tcolorbox}
\DeclareMathSizes{7.5}{7.5}{5}{5}

\definecolor{takeawaybg}{RGB}{235,243,255}
\definecolor{takeawayline}{RGB}{79,129,220}

\newcommand{\projectname}{\emph{\textbf{E\textsuperscript{2}-OPSD}}}

\title{
  E\textsuperscript{2}-OPSD: Taming Entropy Overshoot in\\
  On-Policy Self-Distillation}

\author{%
Yifei Liu$^{1*}$,
Minghao Fang$^{2*}$,
Xinyu Gu$^{3}$,
Chengkai Yao$^{4}$,
Mengdi Liu$^{5}$, \\
\bfseries Tengfei Ma$^{6}$,
Jiangbin Zheng$^{7}$,
Chang Yu$^{8}$,
Zhangyang Gao$^{3\dagger}$ \\[0.6ex]
\normalfont\small $^{1}$The Chinese University of Hong Kong \quad
$^{2}$Zhejiang University \\
\normalfont\small $^{3}$Shanghai Artificial Intelligence Laboratory \quad
$^{4}$Shanghai Jiaotong University \\
\normalfont\small $^{5}$Visual Information Processing and Learning, ICT, CAS \quad
$^{6}$Hunan University \\
\normalfont\small $^{7}$Westlake University \quad
$^{8}$Nanjing University%
}

\hypersetup{pdftitle={E2-OPSD: Taming Entropy Overshoot in On-Policy Self-Distillation},
    pdfauthor={Yifei Liu, Minghao Fang, Xinyu Gu, Chengkai Yao, Mengdi Liu, Tengfei Ma, Jiangbin Zheng, Chang Yu, Zhangyang Gao},pdfsubject={}}

\begin{document}

\maketitle

{%
  \renewcommand{\thefootnote}{\fnsymbol{footnote}}%
  \footnotetext[1]{\mbox{Equal contribution.\quad
    \textsuperscript{\ensuremath{\dagger}}Corresponding author.}}%
}

\begin{abstract}
On-policy self-distillation (OPSD) provides dense token-level supervision without a second model: one network acts as teacher with the reference solution and as student with only the problem. We identify a specific failure mode of this recipe. During training, student token entropy rises past the teacher's and remains elevated, a pattern we call \emph{\textbf{entropy overshoot}}. We trace it to both sides of distillation. The reference-conditioned teacher is confident along its answer-directed reasoning path, but this confidence transfers poorly to student-generated prefixes, making its supervision overly tied to answer-specific cues rather than reusable reasoning patterns; meanwhile, the forward KL used by OPSD continually diffuses the student's predictive distribution without pulling it back. We introduce \projectname{} to address both causes. \textbf{Exemplar-guided teaching} replaces the current answer with a retrieved solved neighboring problem, providing transferable reasoning guidance without revealing the destination and better matching student-reachable states. \textbf{Entropy-aware distillation} uses the student--teacher entropy gap to determine the direction and strength of each token's correction. \projectname{} improves math reasoning by up to 4.3 points in mean@16 over OPSD, while out-of-domain evaluations show gains over the corresponding base models of up to 4.9 points in mean@16 and 5.5 points in pass@8. Despite these gains, \projectname{} remains simple, requiring no additional forward passes or networks.
Code repo: \url{https://github.com/orionfei/E2-OPSD}.
\end{abstract}

\section{Introduction}
\label{sec:introduction}

% Figure 1 is a wrapped inset (text flows around it on the left) to save space.
% wrapfigure must start at a paragraph boundary. It sits one paragraph above the
% first citation so the caption lands inside the running text rather than on the
% page edge; move it down a paragraph if it drifts off the top of the page.
% Both panels are rendered on the same canvas with the same axes rectangle
% (AXES_RECT in the two plotting scripts), so their plot boxes line up here.

\begin{wrapfigure}{r}{0.47\textwidth}
    \centering
    % Produced by plotting_scripts/plot_qwen3_1_7b_entropy.py from
    % raw source data/entropy_length_qwen3_1_7b_opsd_base0811.json.
    \IfFileExists{figures/entropy_qwen3_1_7b.pdf}{%
        \includegraphics[width=0.485\linewidth]{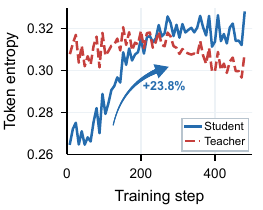}%
    }{\fbox{\rule{0pt}{0.9in}\rule{0.46\linewidth}{0pt}}}%
    \hfill
    % Produced by plotting_scripts/plot_qwen3_1_7b_entropy_bars.py from the
    % teacher-context probe (raw source data/teacher_context_probe_qwen3_1_7b_l57.json).
    \IfFileExists{figures/entropy_bars_qwen3_1_7b.pdf}{%
        \includegraphics[width=0.485\linewidth]{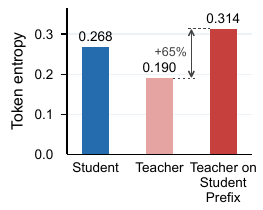}%
    }{\fbox{\rule{0pt}{0.9in}\rule{0.46\linewidth}{0pt}}}
    \caption{\textbf{Left:} Entropy overshoot in Qwen3-1.7B. \textbf{Right:} The reference-conditioned teacher is less certain on student prefixes, suggesting a reasoning-path mismatch.}
    \label{fig:entropy-dynamics}
\end{wrapfigure}

Reinforcement learning with verifiable rewards (RLVR) is now the standard recipe for improving LLM reasoning, yet its supervision remains coarse: an entire rollout collapses into a single scalar reward \citep{shao2024deepseekmath,guo2025deepseek}. 
On-policy distillation (OPD) instead offers dense, token-level supervision by having a teacher score the student's own rollouts \citep{agarwal2024policy,gu2023minillm}, and is now a standard post-training component in Qwen3~\citep{Yang2025Qwen3TR}, GLM-5~\citep{Zeng2026GLM5FV}, and DeepSeek-V4~\citep{DeepSeekAI2026DeepSeekV4TH}. Despite this fine granularity, it still requires a stronger open-weight teacher that shares the student's vocabulary. On-policy self-distillation (OPSD) further removes the need to serve a second model: a copy of the student acts as the teacher, with access to the reference solution, while the student sees only the problem. Teacher and student are evaluated on the same student-generated prefixes and matched with forward KL \citep{zhao2026self}. This makes OPSD an appealing way to turn privileged solution information into dense supervision. In practice, however, its gains are unstable and can deteriorate as training proceeds.

\begin{figure*}[!t]
    \centering
    % Equal heights without clipping: widths follow the two images' aspect ratios
    % ((a) 2172x724 px = 3.00, (b) 1254x1254 px = 1.00), so 0.735 : 0.245 gives both a height of
    % 0.245\textwidth. [b] alignment puts the two sublabels on one baseline.
    \begin{minipage}[b]{0.735\textwidth}
        \centering
        \includegraphics[width=\linewidth]{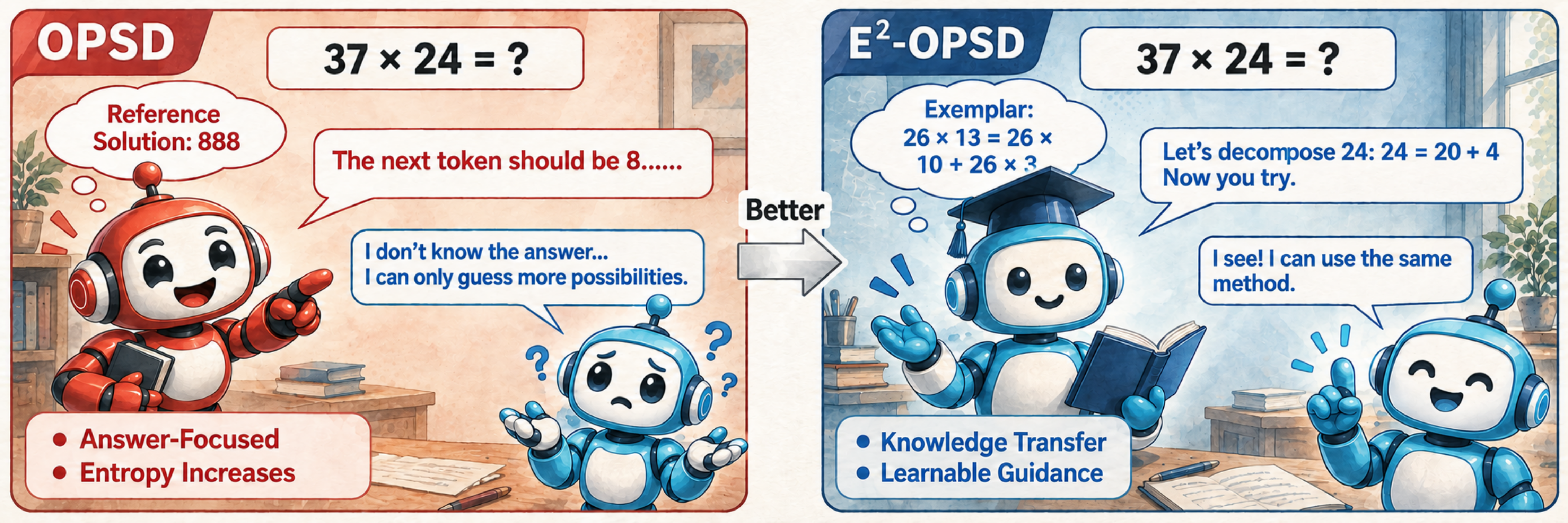}\\[3pt]
        \normalsize (a) Exemplar-Guided Teaching
    \end{minipage}\hfill
    \begin{minipage}[b]{0.245\textwidth}
        \centering
        \includegraphics[width=\linewidth]{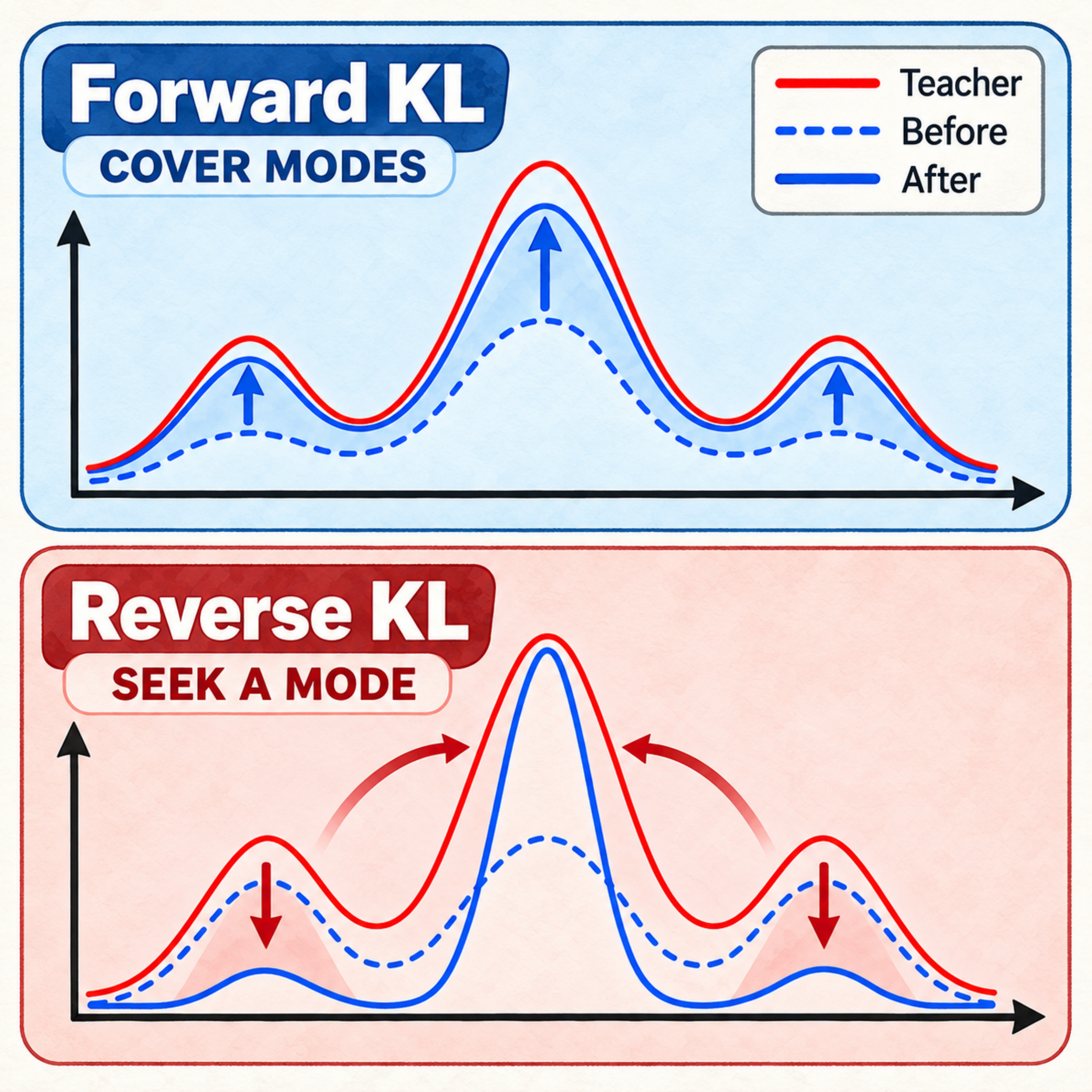}\\[3pt]
        \makebox[\linewidth][r]{\normalsize (b) Forward/Reverse KL}
    \end{minipage}
    \caption{\textbf{(a) Exemplar-Guided Teaching.} Standard OPSD gives the teacher the current reference solution, letting it jump from the student prefix toward the answer. \projectname{} instead uses a solved neighbor as a strategy guide, forcing the teacher to follow the student's intermediate step and solve the current problem itself, turning answer leakage into strategy transfer. \textbf{(b) Forward and reverse KL play complementary roles.} Forward KL promotes exploration over teacher-supported alternatives, while reverse KL removes unsupported probability mass to concentrate the policy. Entropy-aware distillation switches between them.}
    \label{fig:egt-illustration}
\end{figure*}

We identify a characteristic failure mode, as shown in Figure~\ref{fig:entropy-dynamics} (left),  \emph{\textbf{entropy overshoot}}: instead of settling toward the teacher, the student becomes progressively more uncertain, eventually surpassing the teacher's entropy and staying there. Entropy overshoot therefore describes a regime where useful exploration spills into persistent and excessive uncertainty, gradually eroding the student's decisiveness and amplifying teacher-side uncertainty into noisy student behavior.

We examine this failure from both teacher and student perspectives. First, \textbf{OPSD produces supervision mismatched to the student's reasoning path}. The teacher is confident on its own trajectories but becomes substantially more uncertain on student-generated prefixes (Figure~\ref{fig:entropy-dynamics}, right), suggesting that its guidance does not transfer cleanly to the intermediate states the student actually reaches. This mismatch stems from a deeper issue: seeing the current solution lets the teacher narrate toward a known answer rather than reason from scratch, making its supervision answer-specific rather than transferable (Figure~\ref{fig:egt-illustration}(a)). Second, \textbf{OPSD lacks a mechanism to regulate policy spread}. Forward KL is beneficial when the student is too concentrated as it expands probability toward teacher-supported alternatives. But the same pressure persists even after the policy becomes overly diffuse, which makes exploration a one-way ratchet: uncertainty keeps spreading, with no force pulling the policy back into focus (Figure~\ref{fig:egt-illustration}(b)).

These observations motivate \projectname{}, a simple framework that repairs both sides of the learning process with \textbf{E}xemplar-guided teaching and \textbf{E}ntropy-aware distillation, as illustrated in Figure~\ref{fig:framework}. More specifically, \projectname{} makes three contributions:
\begin{itemize}[leftmargin=*]
    \item \textbf{What to learn} --- \emph{Exemplar-guided teaching.} We replace the current problem's reference solution in the teacher's context with a retrieved solved neighboring problem. The exemplar provides a transferable problem-solving reference without revealing the current answer. Conditioned on this reference and the student's generated prefix, the teacher's token-level guidance is meant to be more compatible with the student's ongoing reasoning.
    \item \textbf{How to learn} --- \emph{Entropy-aware distillation.} We use the student--teacher entropy gap as a feedback signal: when the student is more concentrated than the teacher, forward KL encourages exploration; when it becomes more diffuse, reverse KL suppresses unsupported mass and encourages concentration. The gap magnitude downweights uncertain routing decisions and concentrates learning on tokens with larger entropy mismatches. This turns self-distillation from a fixed one-way objective into closed-loop self-balancing control of exploration.
    \item \emph{\textbf{Simple and Generalized.}} \projectname{} adds no extra forward passes during distillation and minimum overhead. Across various model families, sizes from 0.6B to 7B, and both thinking and instruction-tuned models, it mitigates persistent entropy overshoot and improves math reasoning by up to 4.3 points over OPSD, with gains transferring to out-of-domain tasks by up to 4.9 mean@16 points over the corresponding base model. These results suggest that effective self-distillation depends not on maximizing imitation, but on \textbf{aligning teacher guidance with the student's reasoning path and regulating when the student should explore versus concentrate}.
\end{itemize}

\begin{figure*}[t]
    \centering
    \includegraphics[width=\textwidth]{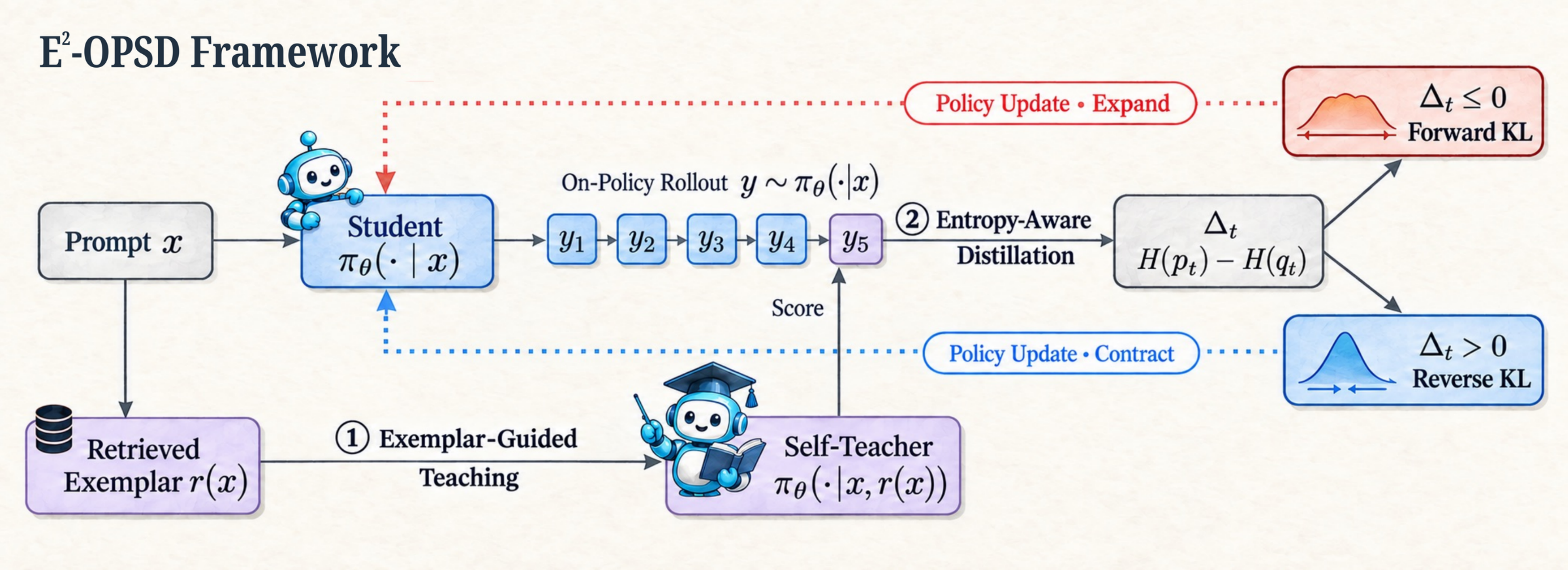}
    \caption{\textbf{Overview of the \projectname{} framework.} Exemplar-guided teaching and entropy-aware distillation jointly improve teacher guidance and regulate student learning.}
    \label{fig:framework}
\end{figure*}

\section{Preliminaries and Related Work}
\label{sec:background}

\subsection{KL Divergence}
\label{sec:kl}

The Kullback--Leibler (KL) divergence is asymmetric, and its two directions
induce distinct approximation biases~\citep{Minka2005DivergenceMA,jin2026entropy}.
For a target distribution $P$ and model distribution $Q$,
\begin{equation}
\mathrm{KL}(P\|Q)
=
\mathbb{E}_{v\sim P}
\left[\log\frac{P(v)}{Q(v)}\right],
\qquad
\mathrm{KL}(Q\|P)
=
\mathbb{E}_{v\sim Q}
\left[\log\frac{Q(v)}{P(v)}\right].
\label{eq:kl-directions}
\end{equation}
Forward KL is \emph{mode covering}: if $P(v)>0$ while $Q(v)\to0$, its
contribution diverges, strongly penalizing target-supported modes that the
model fails to cover. Reverse KL has the complementary \emph{mode-seeking}
bias: if $Q(v)>0$ while $P(v)\to0$, its contribution diverges, strongly
penalizing model mass unsupported by the target~\citep{Minka2005DivergenceMA}.
Thus, forward KL favors coverage, whereas reverse KL favors concentration.

\subsection{On-Policy Distillation}
\label{sec:prelim}

On-policy distillation (OPD) trains a student $\pi_\theta$ to match a teacher
$\pi_T$ on the student's own generations~\citep{agarwal2024policy,gu2023minillm,lu2025onpolicydistillation}.
Given a problem $x\sim\mathcal{D}$, the student samples
$\hat y=(\hat y_1,\ldots,\hat y_{|\hat y|})$, and both models score the same
student-generated prefix $s_t=(x,\hat y_{<t})$:
\begin{equation}
\mathcal{L}(\theta)
=
\mathbb{E}_{x\sim\mathcal{D}}
\mathbb{E}_{\hat y\sim\pi_\theta(\cdot\mid x)}
\left[
\frac{1}{|\hat y|}
\sum_{t=1}^{|\hat y|}
D\!\left(
\pi_T(\cdot\mid s_t)
\,\big\|\,
\pi_\theta(\cdot\mid s_t)
\right)
\right],
\label{eq:opd-objective}
\end{equation}
where $D$ denotes a per-token divergence between the teacher and student
distributions, instantiated as forward KL $\mathrm{KL}(\pi_T\|\pi_\theta)$ or reverse KL
$\mathrm{KL}(\pi_\theta\|\pi_T)$, and gradients flow only through the student. Reverse KL is a common choice~\citep{gu2023minillm,lu2025onpolicydistillation},
while prior work also explores mixed or skewed divergences
as alternative distillation objectives~~\citep{agarwal2024policy,ko2024distillm,Shing2025TAIDTA,wu2025rethinking,jin2026entropy}.
OPSD computes $D$ over the full vocabulary and applies per-token pointwise clipping, capping each vocabulary entry's contribution at $\tau$ so that a few stylistic tokens do not dominate the signal~\citep{zhao2026self}. We keep this pointwise clipping and, following recent sparse-support distillation practice, approximate $D$ on the teacher's top-$K_{\mathrm v}$ tokens together with a tail bucket carrying the remaining probability mass~\citep{li2026rethinkingonpolicydistillationlarge}.
\label{sec:support-set}
\label{sec:kl-clipping}

\subsection{Related Work}
\label{sec:related}

\paragraph{On-policy self-distillation.}
On-policy self-distillation (OPSD) instantiates Eq.~\ref{eq:opd-objective} with a copy of the student as teacher, $\pi_T=\pi_{\theta_0}$, whose context includes the reference solution $y^*$ unseen by the student and is matched via forward KL~\citep{zhao2026self}. This privileged teacher creates an irreducible objective gap~\citep{yang2026self}, overweights stylistic and answer-implied tokens over genuine decisions~\citep{pan2026rlcsd,harne2026privileged}, and suppresses uncertainty important to thinking models~\citep{kaur2026rethinking,kim2026does}. Existing remedies largely patch these effects downstream. \emph{Verifier-directed} methods outsource update control to outcome rewards~\citep{yang2026self,ke2026respecting,he2026self}, effectively reintroducing RLVR. \emph{Signal-purification} methods cancel privileged bias with an additional teacher pass~\citep{pan2026rlcsd,shen2026purified}, paying extra compute after the bias is already introduced. \emph{Context-curation} methods expose less of the solution or replace it with a successful sibling rollout~\citep{han2026adaptive,hubotter2026reinforcement}, but still rely on solution-specific supervision and auxiliary control. These approaches mitigate symptoms without removing their source: teacher guidance tied to the current problem's solution. \projectname{} removes this mismatch at the source, restores the concentrating correction absent from forward KL, and retains a single frozen-teacher pass with no reward.

\paragraph{Entropy-aware post-training.}
Entropy has become a common control signal in post-training, but serves different roles across RLVR and distillation.
In RLVR, it is primarily used to regulate exploration, from preventing entropy collapse to concentrating updates on high-entropy or decision-critical tokens~\citep{cui2025entropy,yu2026dapo,wang2026beyond,Hao2025RethinkingEI}.
In distillation, entropy has been used to guide supervision selection and adaptation: EOPD activates additional forward-KL supervision at high-teacher-entropy tokens, while subsequent work further adapts supervision across uncertainty regimes~\citep{jin2026entropy,zhang2026tailoring,lee2026sead}.
\projectname{} instead targets the relative entropy dynamics of the student and teacher, using their signed token-level entropy gap to correct entropy overshoot during on-policy self-distillation.

\paragraph{Training-time contextual guidance.}
A broad line of work improves post-training by providing auxiliary context that is available during training but absent at inference.
Context distillation internalizes prompts, exemplars, or scratchpads into model weights~\citep{askell2021general,snell2022learning,shin2025generative,kalwar2026masked}.
More recent methods diversify the form of this guidance: OPCD distills experiential knowledge from historical solution traces, SDFT learns on-policy from demonstrations, SKALD conditions a self-teacher on abstract skill cards, and PMD accumulates reusable procedural memories across training episodes~\citep{ye2026policy,shenfeld2026self,jiang2026distill,liu2026procedural}.
Other approaches inject more instance-specific scaffolds directly into optimization, such as partial solutions in QuestA, off-policy reasoning traces in LUFFY, in-context exemplars in XRPO, or retrieved analogous demonstrations in RA-RFT~\citep{li2026questaexpandingreasoningcapacity,yan2025learningreasonoffpolicyguidance,Bamba2025XRPOPT,xiao2026learningreasonanalogyretrievalaugmented}.
OPSD conditions its self-teacher on the current problem's reference solution; \projectname{} instead uses a retrieved solved neighbor, transferring reusable strategy across instances without revealing the current answer.

\section{Method}
\label{sec:method}

\projectname{} repairs both sides of the entropy-overshoot failure identified in Section~\ref{sec:introduction}: \emph{exemplar-guided teaching} changes what the teacher is conditioned on so its guidance is grounded in transferable reasoning rather than answer-specific cues, while \emph{entropy-aware distillation} changes how the student matches it, using the per-token entropy gap to set both the direction and strength of the correction. Both run inside the standard OPSD loop with one frozen teacher and no additional forward passes or extra network.

\paragraph{Setup.}
The student $\pi_\theta$ and frozen teacher $\pi_\phi$, initialized from the same model
($\phi=\theta_0$), score the same student-generated prefix
$\hat y_{<t}$ but under different contexts:
\begin{equation}
p_t(v)=\pi_\theta(v\mid x,\hat y_{<t}),
\qquad
q_t(v;r)=\pi_\phi(v\mid x,r(x),\hat y_{<t}),
\label{eq:rein-policies}
\end{equation}
where $r(x)$ denotes teacher-only privileged context. Standard OPSD sets
$r(x)=y^*$, giving the teacher the reference solution to the current problem.

\subsection{Exemplar-Guided Teaching}
\label{sec:rein-egt}

Conditioning on $y^*$ lets the teacher narrate toward a known answer, producing
guidance that depends on information absent from the student's context.
Exemplar-Guided Teaching (EGT) instead replaces the current solution with a
solved neighboring problem, so the teacher must still solve $x$ itself while
using the exemplar only as transferable strategy
(Figure~\ref{fig:egt-illustration}(a)).

We embed all training problems and compute pairwise cosine similarity offline, keeping each problem's $K$ nearest neighbors (self-excluded) as $\mathcal{N}_K(x)$. For the $g$-th of $G$ rollouts per problem, we draw one neighbor uniformly,
\begin{equation}
    r_g(x) = (x'_g,\, y'^*_g),
    \qquad
    x'_g \sim \mathrm{Uniform}\!\left(\mathcal{N}_K(x)\right),
    \quad x'_g \neq x,
    \label{eq:rein-retrieval}
\end{equation}
seeded deterministically by problem and rollout index, so different rollouts of the same problem see different exemplars at no extra retrieval cost.
Thus the teacher is conditioned on $(x,x'_g,y_g^{\prime *},\hat y_{<t})$,
while the student still sees only $(x,\hat y_{<t})$. EGT therefore removes the
current answer from teacher conditioning and replaces answer-specific privilege
with cross-problem strategy guidance. Prompts and offline preprocessing cost are in
Appendix~\ref{app:method-implementation}.

\subsection{Entropy-Aware Distillation}
\label{sec:rein-ead}

EGT improves \emph{what} the teacher provides; Entropy-Aware Distillation (EAD)
controls \emph{how} the student follows it. Building on the complementary
mode-covering and mode-seeking biases of forward and reverse KL
(Section~\ref{sec:kl}), EAD has two components:
\emph{Bidirectional Entropy Routing} (BER) selects the KL direction from the
student--teacher entropy gap $\Delta_t = H(p_t) - H(q_t)$, while \emph{Adaptive Gap Damping} (AGD) uses its
magnitude to scale the update.

\begin{proposition}[Entropy-gap stability]
If $p_t$ and $q_t$ concentrate their mass on a shared candidate set,
the sign of $\Delta_t$ agrees with that of the gap restricted to this
set whenever $|\Delta_t|$ exceeds a tail-dependent margin
(Appendix~\ref{app:entropy-gap-theory}).
\end{proposition}

\vspace{-0.8em}
\paragraph{Bidirectional entropy routing.}
We use the entropy gap
$\Delta_t$ as a directional proxy for relative policy spread.
Higher entropy corresponds to more effective next-token alternatives.
BER uses this directional signal as follows.
When $\Delta_t\le0$, the student is relatively more concentrated than the
teacher, so BER applies forward KL to broaden coverage. When $\Delta_t>0$,
the student is relatively more diffuse, so BER switches to reverse KL to
refocus the policy (Figure~\ref{fig:egt-illustration}(b)). Using
$\widetilde{\mathcal V}_t$, the teacher's top-$K_{\mathrm v}$ tokens plus a
tail bucket, and the per-element clip $\tau$ inherited from OPSD
(Section~\ref{sec:kl-clipping}),
\begin{equation}
D_t^{\mathrm{BER}} =
\begin{cases}
\sum_{v\in\widetilde{\mathcal V}_t}
\min\!\left(q_t(v)\log\frac{q_t(v)}{p_t(v)},\,\tau\right),
& \Delta_t \le 0,\\[6pt]
\sum_{v\in\widetilde{\mathcal V}_t}
\min\!\left(p_t(v)\log\frac{p_t(v)}{q_t(v)},\,\tau\right),
& \Delta_t > 0.
\end{cases}
\label{eq:rein-routing}
\end{equation}
BER thus replaces OPSD's fixed one-way pressure with a state-dependent
correction that can either broaden or refocus the student policy.

\paragraph{Adaptive gap damping.}
The sign of $\Delta_t$ selects the correction direction, while its magnitude
indicates the need for correction: near-zero gaps make routing sensitive to
small perturbations, whereas larger gaps indicate more pronounced student--teacher
entropy mismatches that call for stronger correction.
We therefore use the gap magnitude as a confidence weight to
downweight uncertain routing decisions and concentrate learning on tokens
with larger entropy mismatches.
Over the valid response tokens $\mathcal{T}$ in the current mini-batch, we
define the typical gap scale
$\sigma^2=\frac{1}{|\mathcal{T}|}\sum_{(i,t)\in\mathcal{T}}\Delta_{i,t}^2$
and weight each routed update by
\begin{equation}
c_t
=
\frac{\Delta_t^2}{\Delta_t^2+\max(\sigma^2,\epsilon^2)}
\in[0,1),
\label{eq:rein-damping}
\end{equation}
where $\epsilon$ is the FP32 machine epsilon, with $c_t=0$ when
$\Delta_t=0$. Gaps small relative to the batch scale
are therefore suppressed, while those much larger than the scale retain
most of their update.
Because $\sigma^2$ is recomputed every mini-batch, the damping adapts to the
training dynamics without a tunable threshold.

\paragraph{Training objective.}
Combining exemplar-guided teaching, bidirectional entropy routing, and adaptive gap
damping, \projectname{} minimizes
\begin{equation}
\mathcal{L}_{\mathcal T}(\theta)
=
\frac{1}{|\mathcal T|}
\sum_{(i,t)\in\mathcal T}
\operatorname{sg}(c_{i,t})\,D_{i,t}^{\mathrm{BER}},
\label{eq:rein-objective}
\end{equation}
where $\mathcal T$ contains the valid response tokens in the current loss
computation, indexed by response $i$ and position $t$. Sampled responses and
exemplars are held fixed during optimization. The teacher conditions on the
sampled solved neighbor $(x',y'^{*})$ through $q_t$, while the student context
remains unchanged. Both the routing
decision and $c_t$ are treated as stop-gradient quantities, so optimization
acts only through the student probabilities in $D_t^{\mathrm{BER}}$.
Compared with OPSD, \projectname{} changes only the teacher context and
per-token divergence, requiring the same rollouts and frozen-teacher scoring
with no additional forward pass, verifier, or reward during distillation. See Appendix~\ref{app:method-implementation} for more overhead analysis. At inference, all privileged context is removed.

\section{Experiments}
\label{sec:experiments}

\subsection{Experimental Settings}
\label{sec:exp-setup}

\paragraph{Models \& Datasets.}
We evaluate four configurations across the Qwen3 \citep{Yang2025Qwen3TR} and
Olmo3 \citep{olmo2026olmo3} families: Qwen3-0.6B and Qwen3-1.7B with thinking enabled,
Qwen3-4B-Instruct-2507, and Olmo-3-7B-Instruct.
The reasoning mode is consistent across the student, teacher (where applicable), and evaluation.
Following \citet{pan2026rlcsd}, we use DeepMath-103K \citep{He2025DeepMath103KAL} for training, with a difficulty band assigned to each
configuration (Appendix~\ref{app:data-config}). We evaluate in-domain mathematical reasoning on
AMC23 \citep{amc23}, AIME'24 \citep{aime24}, AIME'25 \citep{aime25}, and
HMMT'25 Feb \citep{hmmt2025feb}. Out-of-domain evaluation covers logical reasoning
on a fixed, randomly sampled 100-puzzle subset of ZebraLogic~\citep{Lin2025ZebraLogicOT}for computational tractability, and code generation on HumanEval
\citep{Chen2021EvaluatingLL}.

\paragraph{Baselines.}
We compare against (1) \textbf{Base}, the untrained checkpoints;
(2) \textbf{OPSD} \citep{zhao2026self}, which distills a frozen teacher conditioned on the current
problem's reference solution using forward KL; (3) \textbf{SDPO} \citep{hubotter2026reinforcement},
which uses an EMA teacher conditioned on a successful sibling rollout and a generalized JSD loss;
(4) \textbf{GRPO} \citep{shao2024deepseekmath}, which optimizes a clipped policy objective using
group-normalized advantages from verifiable rewards, with KL regularization against a frozen reference policy;
(5) \textbf{RLSD} \citep{yang2026self}, which uses verifiable rewards to determine update directions and
self-distillation signals to modulate token-level update magnitudes;
and (6) \textbf{OPSD + EOPD}, our self-distillation adaptation of EOPD \citep{jin2026entropy}, which combines reverse KL with forward KL at high-teacher-entropy positions. We retain this entropy-gated divergence but replace EOPD's external stronger teacher with OPSD's reference-conditioned self-teacher for a controlled comparison.

\paragraph{Implementation Details.}
We use veRL \citep{Sheng2024HybridFlowAF} with vLLM \citep{Kwon2023EfficientMM} on
8$\times$H200 GPUs and a learning rate of $10^{-6}$.
We faithfully implement each baseline following the algorithmic design
in its original paper, preserving its method-specific objectives and update rules,
including teacher construction for self-distillation methods.
All trained methods and ablations share the same training settings and a budget
of 360 optimizer steps. Detailed baseline configurations, including the explicit
adaptation of EOPD to the OPSD setting, are provided in
Appendix~\ref{app:method-implementation} for reproducibility.
The baseline results in Table~\ref{tab:main} are obtained from these
implementations under the shared protocol.
For self-distillation methods, the teacher is frozen at its initial weights except in SDPO and RLSD.
We use $K = 15$ for a diverse exemplar selection, and supply retrieved exemplars only to the teacher. 
We evaluate each model with three evaluation seeds to reduce noise. We provide the model with only the original questions during evaluation inference. 
See Appendix~\ref{app:implementation} for detailed training and evaluation settings.

\begin{table}[t]
\caption{Main results across model scales and types on math reasoning, reported as mean@16 over three seeds. Green subscripts show our gains over Base, computed from the displayed values. Best values are bold and second-best distinct values underlined. Full evaluation results are in Appendix~\ref{app:full-results}.}
\label{tab:main}
% Convention for Tables 1 and 2: round scores to one decimal first, then subtract
% displayed Ours minus displayed Base for green gains. Compute averages from raw scores.
\centering
\normalsize
\setlength{\tabcolsep}{2.5pt}
\renewcommand{\arraystretch}{0.96}
\setlength{\aboverulesep}{0.3ex}
\setlength{\belowrulesep}{0.35ex}
\begin{tabular}{p{\dimexpr0.21\textwidth-2\tabcolsep\relax}p{\dimexpr0.20\textwidth-2\tabcolsep\relax}*{5}{>{\centering\arraybackslash}p{\dimexpr0.118\textwidth-2\tabcolsep\relax}}}
\toprule
 & & \multicolumn{5}{c}{\textbf{Math Reasoning} (mean@16)} \\
\cmidrule(lr){3-7}
\textbf{Model} & \textbf{Method} & AMC23 & AIME24 & AIME25 & HMMT25 & Avg. \\
\midrule
% Qwen3-0.6B Thinking: level1_3; OPSD, SDPO, and ours use step64; n=16, seeds 17/42/1234.
% Ours math: official math_repeat_20260921; hf_step64; n=16; seeds 17/42/1234.
\multirow{7}{=}{\centering Qwen3-0.6B\\(Thinking)} & Base & $45.8$ & $9.8$ & $14.2$ & $8.3$ & $19.5$ \\
 & OPSD & $45.9$ & $\underline{10.9}$ & $14.2$ & $\mathbf{9.2}$ & $20.0$ \\
 & SDPO & $45.5$ & $9.4$ & $14.6$ & $7.8$ & $19.3$ \\
 & GRPO & $46.1$ & $9.8$ & $\mathbf{16.4}$ & $8.4$ & $\underline{20.2}$ \\
 & RLSD & $\mathbf{46.8}$ & $\mathbf{11.0}$ & $15.1$ & $7.8$ & $\underline{20.2}$ \\
 & OPSD + EOPD & $46.4$ & $9.8$ & $15.9$ & $8.5$ & $20.1$ \\
\rowcolor[gray]{0.85} & \textbf{\mbox{\projectname{} (Ours)}} & $\underline{46.6}_{\textcolor{green!50!black}{+0.8}}$ & $\mathbf{11.0}_{\textcolor{green!50!black}{+1.2}}$ & $\underline{16.2}_{\textcolor{green!50!black}{+2.0}}$ & $\underline{8.9}_{\textcolor{green!50!black}{+0.6}}$ & $\mathbf{20.7}_{\textcolor{green!50!black}{+1.2}}$ \\
\midrule
\multirow{7}{=}{\centering Qwen3-1.7B\\(Thinking)} & Base & $75.5$ & $46.0$ & $34.6$ & $21.9$ & $44.5$ \\
 & OPSD & $74.9$ & $46.0$ & $35.0$ & $\underline{22.7}$ & $44.6$ \\
 & SDPO & $\underline{76.1}$ & $48.0$ & $\underline{35.3}$ & $21.9$ & $\underline{45.3}$ \\
 & GRPO & $74.6$ & $47.0$ & $35.2$ & $21.9$ & $44.7$ \\
 & RLSD & $75.9$ & $\underline{48.3}$ & $34.0$ & $21.8$ & $45.0$ \\
 & OPSD + EOPD & $74.3$ & $47.4$ & $34.2$ & $21.5$ & $44.3$ \\
\rowcolor[gray]{0.85} & \textbf{\mbox{\projectname{} (Ours)}} & $\mathbf{79.0}_{\textcolor{green!50!black}{+3.5}}$ & $\mathbf{52.3}_{\textcolor{green!50!black}{+6.3}}$ & $\mathbf{39.2}_{\textcolor{green!50!black}{+4.6}}$ & $\mathbf{25.0}_{\textcolor{green!50!black}{+3.1}}$ & $\mathbf{48.9}_{\textcolor{green!50!black}{+4.4}}$ \\
\midrule
\multirow{7}{=}{\centering Qwen3-4B-\\Instruct-2507} & Base & $85.4$ & $61.7$ & $46.5$ & $29.0$ & $\underline{55.7}$ \\
 & OPSD & $84.8$ & $\underline{62.4}$ & $45.6$ & $\underline{29.5}$ & $55.6$ \\
 & SDPO & $83.9$ & $61.3$ & $45.8$ & $27.5$ & $54.6$ \\
 & GRPO & $84.4$ & $61.7$ & $45.6$ & $29.1$ & $55.2$ \\
 & RLSD & $85.1$ & $62.2$ & $\underline{46.9}$ & $28.4$ & $\underline{55.7}$ \\
 & OPSD + EOPD & $\underline{85.5}$ & $62.0$ & $46.5$ & $28.4$ & $55.6$ \\
\rowcolor[gray]{0.85} & \textbf{\mbox{\projectname{} (Ours)}} & $\mathbf{87.1}_{\textcolor{green!50!black}{+1.7}}$ & $\mathbf{62.8}_{\textcolor{green!50!black}{+1.1}}$ & $\mathbf{51.0}_{\textcolor{green!50!black}{+4.5}}$ & $\mathbf{30.2}_{\textcolor{green!50!black}{+1.2}}$ & $\mathbf{57.8}_{\textcolor{green!50!black}{+2.1}}$ \\
\midrule
\multirow{7}{=}{\centering Olmo-3-7B-\\Instruct} & Base & $80.5$ & $52.8$ & $37.8$ & $23.3$ & $48.6$ \\
 & OPSD & $80.5$ & $52.8$ & $38.3$ & $22.4$ & $48.5$ \\
 & SDPO & $80.3$ & $51.3$ & $\underline{39.9}$ & $\mathbf{24.5}$ & $49.0$ \\
 & GRPO & $\underline{80.9}$ & $52.6$ & $38.6$ & $23.5$ & $48.9$ \\
 & RLSD & $80.4$ & $\underline{53.6}$ & $38.8$ & $23.8$ & $\underline{49.2}$ \\
 & OPSD + EOPD & $79.6$ & $53.5$ & $39.2$ & $\underline{24.0}$ & $49.1$ \\
\rowcolor[gray]{0.85} & \textbf{\mbox{\projectname{} (Ours)}} & $\mathbf{81.2}_{\textcolor{green!50!black}{+0.7}}$ & $\mathbf{56.3}_{\textcolor{green!50!black}{+3.5}}$ & $\mathbf{40.0}_{\textcolor{green!50!black}{+2.2}}$ & $23.9_{\textcolor{green!50!black}{+0.6}}$ & $\mathbf{50.3}_{\textcolor{green!50!black}{+1.7}}$ \\
\bottomrule
\end{tabular}
\end{table}

% Table 2 sits directly under Table 1; both are top floats on the same page.
\begin{table}[!t]
\caption{\textbf{Out-of-domain evaluation on logical reasoning and code generation.}
\projectname{} consistently improves both mean@16 and pass@8 across ZebraLogic and
HumanEval, showing that its gains extend beyond in-domain mathematical reasoning.}
\label{tab:ood}
\centering
\begin{minipage}{0.86\textwidth}
\centering
{\normalsize
\setlength{\tabcolsep}{1.5pt}
\renewcommand{\arraystretch}{0.98}
\setlength{\aboverulesep}{0.15ex}
\setlength{\belowrulesep}{0.2ex}

\begin{tabular}{
p{\dimexpr0.20\linewidth-2\tabcolsep\relax}
p{\dimexpr0.14\linewidth-2\tabcolsep\relax}
*{4}{>{\centering\arraybackslash}p{\dimexpr0.165\linewidth-2\tabcolsep\relax}}
}
\toprule
& & \multicolumn{2}{c}{\textbf{Logic: ZebraLogic}} & \multicolumn{2}{c}{\textbf{Coding: HumanEval}} \\
\cmidrule(lr){3-4}\cmidrule(lr){5-6}
\textbf{Model} & \textbf{Method}
& \makebox[2em][c]{mean@16}\makebox[1.8em][l]{} & \makebox[2em][c]{pass@8}\makebox[1.8em][l]{}
& \makebox[2em][c]{mean@16}\makebox[1.8em][l]{} & \makebox[2em][c]{pass@8}\makebox[1.8em][l]{} \\
\midrule

& Base
& \makebox[2em][r]{$28.8$}\makebox[1.8em][l]{}
& \makebox[2em][r]{$56.5$}\makebox[1.8em][l]{}
& \makebox[2em][r]{$52.6$}\makebox[1.8em][l]{}
& \makebox[2em][r]{$78.4$}\makebox[1.8em][l]{} \\
\rowcolor[gray]{0.85}
\multirow{-2}{=}{\centering Qwen3-0.6B\\(Thinking)}
& \mbox{\projectname{}}
& \makebox[2em][r]{$\mathbf{29.0}$}\makebox[1.8em][l]{$_{\textcolor{green!50!black}{+0.2}}$}
& \makebox[2em][r]{$56.5$}\makebox[1.8em][l]{$_{\textcolor{green!50!black}{+0.0}}$}
& \makebox[2em][r]{$\mathbf{53.1}$}\makebox[1.8em][l]{$_{\textcolor{green!50!black}{+0.5}}$}
& \makebox[2em][r]{$\mathbf{79.0}$}\makebox[1.8em][l]{$_{\textcolor{green!50!black}{+0.6}}$} \\
\midrule

& Base
& \makebox[2em][r]{$62.5$}\makebox[1.8em][l]{}
& \makebox[2em][r]{$78.3$}\makebox[1.8em][l]{}
& \makebox[2em][r]{$83.8$}\makebox[1.8em][l]{}
& \makebox[2em][r]{$93.2$}\makebox[1.8em][l]{} \\
\rowcolor[gray]{0.85}
\multirow{-2}{=}{\centering Qwen3-1.7B\\(Thinking)}
& \mbox{\projectname{}}
& \makebox[2em][r]{$\mathbf{67.4}$}\makebox[1.8em][l]{$_{\textcolor{green!50!black}{+4.9}}$}
& \makebox[2em][r]{$\mathbf{83.8}$}\makebox[1.8em][l]{$_{\textcolor{green!50!black}{+5.5}}$}
& \makebox[2em][r]{$\mathbf{84.5}$}\makebox[1.8em][l]{$_{\textcolor{green!50!black}{+0.7}}$}
& \makebox[2em][r]{$\mathbf{93.8}$}\makebox[1.8em][l]{$_{\textcolor{green!50!black}{+0.6}}$} \\
\midrule

& Base
& \makebox[2em][r]{$79.1$}\makebox[1.8em][l]{}
& \makebox[2em][r]{$92.4$}\makebox[1.8em][l]{}
& \makebox[2em][r]{$84.0$}\makebox[1.8em][l]{}
& \makebox[2em][r]{$89.9$}\makebox[1.8em][l]{} \\
\rowcolor[gray]{0.85}
\multirow{-2}{=}{\centering Qwen3-4B\\Instruct-2507}
& \mbox{\projectname{}}
& \makebox[2em][r]{$\mathbf{82.2}$}\makebox[1.8em][l]{$_{\textcolor{green!50!black}{+3.1}}$}
& \makebox[2em][r]{$\mathbf{92.8}$}\makebox[1.8em][l]{$_{\textcolor{green!50!black}{+0.4}}$}
& \makebox[2em][r]{$\mathbf{84.8}$}\makebox[1.8em][l]{$_{\textcolor{green!50!black}{+0.8}}$}
& \makebox[2em][r]{$\mathbf{90.7}$}\makebox[1.8em][l]{$_{\textcolor{green!50!black}{+0.8}}$} \\
\midrule

& Base
& \makebox[2em][r]{$32.7$}\makebox[1.8em][l]{}
& \makebox[2em][r]{$56.6$}\makebox[1.8em][l]{}
& \makebox[2em][r]{$78.2$}\makebox[1.8em][l]{}
& \makebox[2em][r]{$92.1$}\makebox[1.8em][l]{} \\
\rowcolor[gray]{0.85}
\multirow{-2}{=}{\centering OLMo-3-7B\\Instruct}
& \mbox{\projectname{}}
& \makebox[2em][r]{$\mathbf{33.9}$}\makebox[1.8em][l]{$_{\textcolor{green!50!black}{+1.2}}$}
& \makebox[2em][r]{$\mathbf{57.9}$}\makebox[1.8em][l]{$_{\textcolor{green!50!black}{+1.3}}$}
& \makebox[2em][r]{$\mathbf{78.7}$}\makebox[1.8em][l]{$_{\textcolor{green!50!black}{+0.5}}$}
& \makebox[2em][r]{$\mathbf{92.5}$}\makebox[1.8em][l]{$_{\textcolor{green!50!black}{+0.4}}$} \\
\bottomrule
\end{tabular}
}
\end{minipage}
\end{table}

\subsection{Main Results}
\label{sec:exp-main}

\paragraph{\projectname{} delivers robust and transferable gains.}
As shown in Tables~\ref{tab:main} and~\ref{tab:ood}, \projectname{} achieves the highest average math mean@16 across all four model configurations, spanning two model families and both thinking and instruct modes. In particular, \projectname{} improves over Base by 4.4 and 2.1 percentage points on Qwen3-1.7B Thinking and Qwen3-4B-Instruct-2507, respectively, while leading on all four math benchmarks for both models.
 Its consistent effectiveness across distinct reasoning modes suggests that \projectname{} addresses a fundamental failure mode of OPSD, rather than relying on a model- or reasoning-format-specific recipe. More importantly, the gains extend beyond math reasoning: \projectname{} improves over Base in both logical reasoning and code generation across all four configurations, with logical reasoning gains of 4.9 and 3.1 points on Qwen3-1.7B Thinking and Qwen3-4B-Instruct-2507, respectively, alongside consistent gains in code generation as well. Together, these gains support our central premise that effective self-distillation should transfer reasoning strategies and regulate distillation dynamics, rather than merely imitate solutions.

\begin{figure*}[t]
    \centering
    \begin{minipage}[t]{0.30\textwidth}
        \centering
        % Figure 5 is built by plotting_scripts/fig05_training_dynamics.py
        % from figure_data/; see plotting_scripts/README.md.
        \IfFileExists{figures/entropy_opsd_vs_ours_olmo_3_7b_instruct.pdf}{%
            \includegraphics[width=\linewidth]{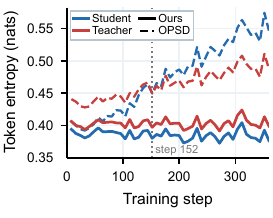}%
        }{\fbox{\rule{0pt}{1.45in}\rule{0.96\linewidth}{0pt}}}
        \\(a) Entropy dynamics on Olmo-3-7B-Instruct.
    \end{minipage}\hfill
    \begin{minipage}[t]{0.30\textwidth}
        \centering
        \IfFileExists{figures/entropy_opsd_vs_ours_qwen3_4b_instruct.pdf}{%
            \includegraphics[width=\linewidth]{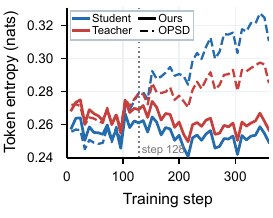}%
        }{\fbox{\rule{0pt}{1.45in}\rule{0.96\linewidth}{0pt}}}
        \\(b) Entropy dynamics on Qwen3-4B-Instruct-2507.
    \end{minipage}\hfill
    \begin{minipage}[t]{0.30\textwidth}
        \centering
        \IfFileExists{figures/kl_direction_qwen3_1_7b.pdf}{%
            \includegraphics[width=\linewidth]{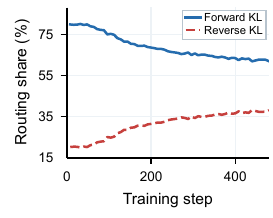}%
        }{\fbox{\rule{0pt}{1.45in}\rule{0.96\linewidth}{0pt}}}
        \\(c) Emergent KL-routing dynamics.
    \end{minipage}
\caption{\textbf{Training dynamics.}
(a,b) Student and teacher token entropy under OPSD (dashed) and \projectname{} (solid);
the dotted line marks the first student--teacher entropy crossing under OPSD.
(c) On Qwen3-1.7B Thinking, routing share denotes the fraction of response tokens
assigned to forward or reverse KL by BER. The routing naturally shifts from
forward-dominated exploration toward increased reverse-KL concentration as training proceeds.}
\label{fig:dynamics}
\end{figure*}

\paragraph{\projectname{} mitigates persistent entropy overshoot.}
Figures~\ref{fig:dynamics}(a) and~\ref{fig:dynamics}(b) revisit the failure identified in the introduction. Under OPSD, student entropy crosses the teacher early and continues to drift upward, creating a persistent entropy gap. In contrast, \projectname{} keeps student entropy tightly regulated around the teacher throughout training, preventing runaway diffusion while preserving sufficient exploration. This curbs entropy overshoot and stabilizes distillation.
\flushbottom

\paragraph{\projectname{} exhibits an emergent exploration-to-concentration transition.}
As shown in Figure~\ref{fig:dynamics}(c), forward KL dominates early, promoting exploration toward teacher-supported alternatives, while reverse KL becomes increasingly prevalent later, suppressing excess probability mass and concentrating the student distribution. This shift emerges without an explicit schedule and naturally tracks the evolving student state. As both routing curves flatten, \projectname{} further exhibits a self-equilibrating diffusion--concentration dynamic.

% One figure, three panels, drawn on a single canvas by plotting_scripts/fig05_ablation_row.py
% at exactly \textwidth, so the three plot boxes share one baseline and one top edge and the
% sublabels are part of the graphic.
\begin{figure*}[t]
    \centering

\begin{minipage}[t]{0.32\textwidth}
    \centering
    \includegraphics[width=\linewidth]{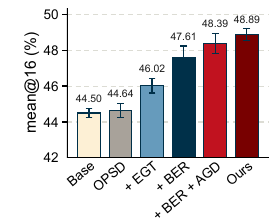}
    \\(a) Component contributions.
\end{minipage}
\hfill
\begin{minipage}[t]{0.32\textwidth}
    \centering
    \includegraphics[width=\linewidth]{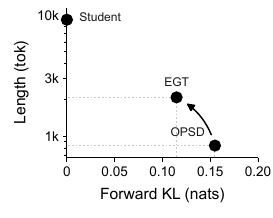}
    \\(b) Teacher-context probe.
\end{minipage}
\hfill
\begin{minipage}[t]{0.32\textwidth}
    \centering
    \includegraphics[width=\linewidth]{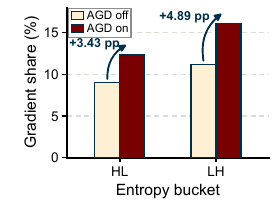}
    \\(c) AGD gradient probe.
\end{minipage}

\caption{\textbf{Component ablation and mechanism probes.}
(a) Performance contribution of each design component; ``+'' denotes the component added on top of OPSD.
(b) EGT improves teacher--student compatibility. The x-axis denotes teacher-to-student forward KL evaluated on student prefixes; the y-axis shows response length from each setting's own rollouts (log scale). (c) AGD sharpens entropy-aware routing. HL denotes high-student/low-teacher
entropy and LH the reverse.}
    \label{fig:ablation}
\end{figure*}

\subsection{Ablation Study}
\label{sec:exp-ablation}

\paragraph{Performance breakdown by design components.}
Figure~\ref{fig:ablation}(a) isolates the contribution of each design choice on Qwen3-1.7B Thinking. Starting from OPSD at 44.6 macro mean@16, exemplar-guided teaching raises performance to 46.0, while bidirectional entropy routing reaches 47.6. Adding adaptive gap damping further improves the score to 48.4, and the full combination achieves 48.9. The steady gains suggest that the components provide complementary benefits, improving both the supervision the student receives and how that supervision is applied. For detailed results, please refer to Table~\ref{tab:app-ablation-mean16}.

\paragraph{State-dependent KL routing outperforms fixed directions.}
With the reference-answer teacher fixed and adaptive gap damping disabled, our entropy-aware routing clearly outperforms all other methods that apply fixed KL directions (see Figure~\ref{fig:ber-direction-dynamics}). To be more precise, forward KL keeps expanding the student and eventually drives entropy overshoot, making the student overly diffuse. Reverse KL suppresses excess mass but also restricts exploration when the student has not yet covered enough plausible alternatives. Their fixed mixture softens both extremes, yet applies expansion and contraction simultaneously regardless of the student's current state, causing the two corrections to partially counteract each other rather than selecting the one that is needed. In contrast, \projectname{} uses the student--teacher entropy gap to select the needed correction, expanding when the student is too concentrated and contracting when it is too diffuse, yielding more stable entropy dynamics and stronger validation performance.

% KL-direction dynamics, full width; drawn by plotting_scripts/fig07_ber_direction_dynamics.py.
\begin{figure*}[!t]
\centering
\includegraphics[width=0.92\textwidth]{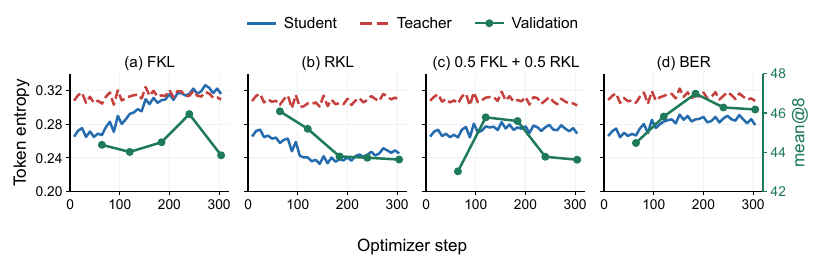}
\caption{\textbf{State-dependent KL routing avoids fixed directional bias.}
KL-direction dynamics on Qwen3-1.7B Thinking with the teacher context fixed
and AGD disabled. The x-axis shows optimizer steps; the left y-axis reports
student and teacher token entropy, while the right y-axis of panel (d) reports validation
mean@8. Panels compare (a) forward KL, (b) reverse KL, (c) an equal-weight
forward/reverse KL mixture, and (d) BER. The trajectories reveal distinct
distributional dynamics under fixed versus state-dependent KL routing.}
\label{fig:ber-direction-dynamics}
\end{figure*}

\paragraph{Exemplar guidance makes the teacher locally learnable.}
Figure~\ref{fig:ablation}(b) reveals why replacing the current reference solution matters.
When conditioned on the reference answer, the teacher produces much shorter and more confident trajectories than the student, with an entropy of 0.190 nats versus the student's 0.268 nats.
EGT instead keeps the teacher closer to the student's reasoning regime, at 2,090 tokens and 0.261 nats.
Moreover, EGT reduces teacher--student forward KL by 25.9\% (0.1549 to 0.1148 nats) on the same student prefixes.
Together, these observations suggest that EGT brings the teacher's native uncertainty closer to the student's reasoning regime while providing more closely aligned token-level guidance on student-visited states.

\paragraph{AGD sharpens entropy-aware routing.}
Figure~\ref{fig:ablation}(c) isolates how AGD redistributes token-level learning
signals. Here, high entropy means $H\ge\ln 2$, corresponding to at least two
effective next-token choices. Entropy-discordant positions---HL and LH, where
student and teacher fall into opposite high/low entropy regimes---account for
only 7.67\% of all response tokens, yet contribute 20.18\% of the relative
gradient mass without AGD.
Enabling AGD further raises this share to 28.50\%, with both HL and LH increasing
consistently. This shows that AGD does not uniformly rescale the distillation
objective; instead, it makes BER more selective by shifting learning toward
positions with clearer student--teacher disagreement in policy spread. To rule out that small entropy gaps merely hide distributional mismatch, we measure probability-mass overlap: strongly damped positions have 0.985 mean overlap (Appendix~\ref{app:mechanism-probes}), supporting the entropy gap as a simple yet principled proxy for local distributional mismatch.
In this sense, BER determines \emph{which direction} to correct,
while AGD determines \emph{which routing decisions deserve stronger updates}.

\section{Conclusion}
\label{sec:conclusion}
We identify entropy overshoot as a key failure mode of on-policy self-distillation: a teacher that has read the answer supplies guidance the student cannot fully learn, and OPSD's forward-KL objective keeps pushing the student outward, turning useful exploration into persistent distributional mismatch.
We introduce \projectname{}, which controls this process by constructing more locally learnable teacher targets and adapting both the direction and the strength of distillation to the student's training state.
Across both thinking and instruct models, \projectname{} consistently improves model's reasoning capability, with gains transferring to out-of-domain tasks.
Our analyses show that effective self-distillation requires matching teacher guidance to student-reachable states and uncertainty structure, rather than simply increasing or suppressing entropy. These results suggest that controlling where exploration occurs is as important as encouraging exploration itself.

\bibliography{references}

@article{guo2025deepseek,
  title={Deepseek-r1: Incentivizing reasoning capability in llms via reinforcement learning},
  author={Guo, Daya and Yang, Dejian and Zhang, Haowei and Song, Junxiao and Wang, Peiyi and Zhu, Qihao and Xu, Runxin and Zhang, Ruoyu and Ma, Shirong and Bi, Xiao and others},
  journal={arXiv preprint arXiv:2501.12948},
  year={2025}
}

@article{lu2025onpolicydistillation,
  author = {Kevin Lu and Thinking Machines Lab},
  title = {On-Policy Distillation},
  journal = {Thinking Machines Lab: Connectionism},
  year = {2025},
  note = {https://thinkingmachines.ai/blog/on-policy-distillation},
  doi = {10.64434/tml.20251026},
}

@misc{yu2026dapo,
      title={DAPO: An Open-Source LLM Reinforcement Learning System at Scale}, 
      author={Qiying Yu and Zheng Zhang and Ruofei Zhu and Yufeng Yuan and Xiaochen Zuo and Yu Yue and Weinan Dai and Tiantian Fan and Gaohong Liu and Lingjun Liu and Xin Liu and Haibin Lin and Zhiqi Lin and Bole Ma and Guangming Sheng and Yuxuan Tong and Chi Zhang and Mofan Zhang and Wang Zhang and Hang Zhu and Jinhua Zhu and Jiaze Chen and Jiangjie Chen and Chengyi Wang and Hongli Yu and Yuxuan Song and Xiangpeng Wei and Hao Zhou and Jingjing Liu and Wei-Ying Ma and Ya-Qin Zhang and Lin Yan and Mu Qiao and Yonghui Wu and Mingxuan Wang},
      year={2025},
      eprint={2503.14476},
      archivePrefix={arXiv},
      primaryClass={cs.LG},
      url={https://arxiv.org/abs/2503.14476}, 
}

@article{gu2023minillm,
  title={Minillm: On-policy distillation of large language models},
  author={Gu, Yuxian and Dong, Li and Wei, Furu and Huang, Minlie},
  journal={arXiv preprint arXiv:2306.08543},
  year={2023}
}

@article{ko2024distillm,
  title={Distillm: Towards streamlined distillation for large language models},
  author={Ko, Jongwoo and Kim, Sungnyun and Chen, Tianyi and Yun, Se-Young},
  journal={arXiv preprint arXiv:2402.03898},
  year={2024}
}

@inproceedings{wu2025rethinking,
  title={Rethinking kullback-leibler divergence in knowledge distillation for large language models},
  author={Wu, Taiqiang and Tao, Chaofan and Wang, Jiahao and Yang, Runming and Zhao, Zhe and Wong, Ngai},
  booktitle={Proceedings of the 31st International Conference on Computational Linguistics},
  pages={5737--5755},
  year={2025}
}

@article{jin2026entropy,
  title={Entropy-aware on-policy distillation of language models},
  author={Jin, Woogyeol and Min, Taywon and Yang, Yongjin and Wei, Dennis and Zhou, Yi and Kadhe, Swanand Ravindra and Baracaldo, Nathalie and Lee, Kimin},
  journal={arXiv preprint arXiv:2603.07079},
  year={2026}
}

@article{ye2026policy,
  title={On-policy context distillation for language models},
  author={Ye, Tianzhu and Dong, Li and Wu, Xun and Huang, Shaohan and Wei, Furu},
  journal={arXiv preprint arXiv:2602.12275},
  year={2026}
}

@article{shenfeld2026self,
  title={Self-distillation enables continual learning},
  author={Shenfeld, Idan and Damani, Mehul and H{\"u}botter, Jonas and Agrawal, Pulkit},
  journal={arXiv preprint arXiv:2601.19897},
  year={2026}
}

@article{hubotter2026reinforcement,
  title={Reinforcement learning via self-distillation},
  author={H{\"u}botter, Jonas and L{\"u}beck, Frederike and Behric, Lejs and Baumann, Anton and Bagatella, Marco and Marta, Daniel and Hakimi, Ido and Shenfeld, Idan and Buening, Thomas Kleine and Guestrin, Carlos and others},
  journal={arXiv preprint arXiv:2601.20802},
  year={2026}
}

@article{yang2026self,
  title={Self-distilled rlvr},
  author={Yang, Chenxu and Qin, Chuanyu and Si, Qingyi and Chen, Minghui and Gu, Naibin and Yao, Dingyu and Lin, Zheng and Wang, Weiping and Wang, Jiaqi and Duan, Nan},
  journal={arXiv preprint arXiv:2604.03128},
  year={2026}
}

@article{pan2026rlcsd,
  title={RLCSD: Reinforcement Learning with Contrastive On-Policy Self-Distillation},
  author={Pan, Leyi and Tao, Shuchang and Zhai, Yunpeng and Zhang, Lingzhe and Liu, Zhaoyang and Ding, Bolin and Liu, Aiwei and Wen, Lijie},
  journal={arXiv preprint arXiv:2606.11709},
  year={2026}
}

@article{ke2026respecting,
  title={Respecting self-uncertainty in on-policy self-distillation for efficient llm reasoning},
  author={Ke, Junlong and Wen, Zichen and Li, Weijia and He, Conghui and Zhang, Linfeng},
  journal={arXiv preprint arXiv:2605.13255},
  year={2026}
}

@article{han2026adaptive,
  title={Adaptive Teacher Exposure for Self-Distillation in LLM Reasoning},
  author={Han, Zihao and Zhang, Tiangang and Wang, Huaibin and Sun, Yilun},
  journal={arXiv preprint arXiv:2605.11458},
  year={2026}
}

@article{kaur2026rethinking,
  title={Rethinking On-Policy Self-Distillation for Thinking Models},
  author={Kaur, Simran and Ri, Narutatsu and He, Yinghui and Fowl, Liam and Arora, Sanjeev},
  journal={arXiv preprint arXiv:2607.05184},
  year={2026}
}

@article{he2026self,
  title={Self-distillation zero: Self-revision turns binary rewards into dense supervision},
  author={He, Yinghui and Kaur, Simran and Bhaskar, Adithya and Yang, Yongjin and Liu, Jiarui and Ri, Narutatsu and Fowl, Liam and Panigrahi, Abhishek and Chen, Danqi and Arora, Sanjeev},
  journal={arXiv preprint arXiv:2604.12002},
  year={2026}
}

@article{cui2025entropy,
  title={The entropy mechanism of reinforcement learning for reasoning language models},
  author={Cui, Ganqu and Zhang, Yuchen and Chen, Jiacheng and Yuan, Lifan and Wang, Zhi and Zuo, Yuxin and Li, Haozhan and Fan, Yuchen and Chen, Huayu and Chen, Weize and others},
  journal={arXiv preprint arXiv:2505.22617},
  year={2025}
}

@misc{wang2026beyond,
      title={Beyond the 80/20 Rule: High-Entropy Minority Tokens Drive Effective Reinforcement Learning for LLM Reasoning}, 
      author={Shenzhi Wang and Le Yu and Chang Gao and Chujie Zheng and Shixuan Liu and Rui Lu and Kai Dang and Xionghui Chen and Jianxin Yang and Zhenru Zhang and Yuqiong Liu and An Yang and Andrew Zhao and Yang Yue and Shiji Song and Bowen Yu and Gao Huang and Junyang Lin},
      year={2025},
      eprint={2506.01939},
      archivePrefix={arXiv},
      primaryClass={cs.CL},
      url={https://arxiv.org/abs/2506.01939}, 
}

@article{shen2026purified,
  title={Purified OPSD: On-Policy Self-Distillation Without Losing How to Think},
  author={Shen, Zhanming and Tong, Jintao and Yan, Shaotian and Shen, Chen and Chen, Hao and Ye, Wentao and Hu, Xiaomeng and Miao, Rui and Wang, Haobo and Zhao, Junbo and others},
  journal={arXiv preprint arXiv:2607.02234},
  year={2026}
}

@article{zhang2026tailoring,
  title={Tailoring Teaching to Aptitude: Direction-Adaptive Self-Distillation for LLM Reasoning},
  author={Zhang, Hongbin and Wang, Chaozheng and Chen, Kehai and Pan, Youcheng and Xiang, Yang and Wang, Jinpeng and Zhang, Min},
  journal={arXiv preprint arXiv:2605.22263},
  year={2026}
}

@article{kim2026does,
  title={Why Does Self-Distillation (Sometimes) Degrade the Reasoning Capability of LLMs?},
  author={Kim, Jeonghye and Luo, Xufang and Kim, Minbeom and Lee, Sangmook and Kim, Dohyung and Jeon, Jiwon and Li, Dongsheng and Yang, Yuqing},
  journal={arXiv preprint arXiv:2603.24472},
  year={2026}
}

@article{harne2026privileged,
  title={Privileged, but Biased: How PI-Conditioned Teachers Break Self-Distillation},
  author={Harne, Sarthak and Karkar, Chinmay and Pandya, Yash and Awadallah, Ahmed and Nambi, Akshay},
  journal={arXiv preprint arXiv:2608.04794},
  year={2026}
}

@article{lee2026sead,
  title={SEAD: Competence-Aware On-Policy Distillation via Entropy-Guided Supervision},
  author={Lee, Chia-Hsuan and Cheng, Zelei and Wang, Yu and Ni, Renkun and Sahu, Sambit and Zhang, Shi-Xiong and Campbell, William},
  journal={arXiv preprint arXiv:2606.28562},
  year={2026}
}

@article{jiang2026distill,
  title={Distill Skills into Weights, Not Prompts: Abstract Skills as Privileged Signals for On-Policy Self-Distillation},
  author={Jiang, Yubo and Xie, Fengying and Jiang, Zhiguo and Zhang, Haopeng},
  journal={arXiv preprint arXiv:2608.09826},
  year={2026}
}

@article{liu2026procedural,
  title={Procedural Memory Distillation: Online Reflection for Self-Improving Language Models},
  author={Liu, Ye and Bansal, Srijan and Pang, Bo and Li, Yang and Liu, Zeyu Leo and Ming, Yifei and Ke, Zixuan and Joty, Shafiq and Yavuz, Semih},
  journal={arXiv preprint arXiv:2607.01480},
  year={2026}
}

@inproceedings{agarwal2024policy,
  title={On-policy distillation of language models: Learning from self-generated mistakes},
  author={Agarwal, Rishabh and Vieillard, Nino and Zhou, Yongchao and Stanczyk, Piotr and Ramos Garea, Sabela and Geist, Matthieu and Bachem, Olivier},
  booktitle={International Conference on Learning Representations},
  volume={2024},
  pages={21246--21263},
  year={2024}
}

@article{shao2024deepseekmath,
  title={Deepseekmath: Pushing the limits of mathematical reasoning in open language models},
  author={Shao, Zhihong and Wang, Peiyi and Zhu, Qihao and Xu, Runxin and Song, Junxiao and Bi, Xiao and Zhang, Haowei and Zhang, Mingchuan and Li, YK and Wu, Yang and others},
  journal={arXiv preprint arXiv:2402.03300},
  year={2024}
}

@article{zhao2026self,
  title={Self-distilled reasoner: On-policy self-distillation for large language models},
  author={Zhao, Siyan and Xie, Zhihui and Liu, Mengchen and Huang, Jing and Pang, Guan and Chen, Feiyu and Grover, Aditya},
  journal={arXiv preprint arXiv:2601.18734},
  year={2026}
}

@article{askell2021general,
  title={A general language assistant as a laboratory for alignment},
  author={Askell, Amanda and Bai, Yuntao and Chen, Anna and Drain, Dawn and Ganguli, Deep and Henighan, Tom and Jones, Andy and Joseph, Nicholas and Mann, Ben and DasSarma, Nova and others},
  journal={arXiv preprint arXiv:2112.00861},
  year={2021}
}

@article{snell2022learning,
  title={Learning by distilling context},
  author={Snell, Charlie and Klein, Dan and Zhong, Ruiqi},
  journal={arXiv preprint arXiv:2209.15189},
  year={2022}
}

@inproceedings{shin2025generative,
  title={Generative prompt internalization},
  author={Shin, Haebin and Ji, Lei and Gong, Yeyun and Kim, Sungdong and Choi, Eunbi and Seo, Minjoon},
  booktitle={Proceedings of the 2025 Conference of the Nations of the Americas Chapter of the Association for Computational Linguistics: Human Language Technologies (Volume 1: Long Papers)},
  pages={7338--7363},
  year={2025}
}

@article{kalwar2026masked,
  title={Masked Distillation: Internalizing the Chain-of-Thought in Language Models},
  author={Kalwar, Durgesh and Palod, Vardhan and Kambhampati, Subbarao},
  journal={arXiv preprint arXiv:2607.22629},
  year={2026}
}

@misc{li2026questaexpandingreasoningcapacity,
      title={QuestA: Expanding Reasoning Capacity in LLMs via Question Augmentation}, 
      author={Jiazheng Li and Hongzhou Lin and Hong Lu and Kaiyue Wen and Zaiwen Yang and Jiaxuan Gao and Yi Wu and Jingzhao Zhang},
      year={2026},
      eprint={2507.13266},
      archivePrefix={arXiv},
      primaryClass={cs.CL},
      url={https://arxiv.org/abs/2507.13266}, 
}

@misc{yan2025learningreasonoffpolicyguidance,
      title={Learning to Reason under Off-Policy Guidance}, 
      author={Jianhao Yan and Yafu Li and Zican Hu and Zhi Wang and Ganqu Cui and Xiaoye Qu and Yu Cheng and Yue Zhang},
      year={2025},
      eprint={2504.14945},
      archivePrefix={arXiv},
      primaryClass={cs.LG},
      url={https://arxiv.org/abs/2504.14945}, 
}

@misc{xiao2026learningreasonanalogyretrievalaugmented,
      title={Learning to Reason by Analogy via Retrieval-Augmented Reinforcement Fine-Tuning}, 
      author={Zilin Xiao and Qi Ma and Chun-cheng Jason Chen and Xintao Chen and Avinash Atreya and Hanjie Chen and Vicente Ordonez},
      year={2026},
      eprint={2606.13680},
      archivePrefix={arXiv},
      primaryClass={cs.CL},
      url={https://arxiv.org/abs/2606.13680}, 
}

@article{Bamba2025XRPOPT,
  title={XRPO: Pushing the limits of GRPO with Targeted Exploration and Exploitation},
  author={Udbhav Bamba and Ming-Hao Fang and Yi-Fan Yu and Haizhong Zheng and Fan Lai},
  journal={ArXiv},
  year={2025},
  volume={abs/2510.06672},
  url={https://api.semanticscholar.org/CorpusID:281891555}
}

@inproceedings{Minka2005DivergenceMA,
  title={Divergence measures and message passing},
  author={Thomas P. Minka},
  year={2005},
  url={https://api.semanticscholar.org/CorpusID:7585417}
}

@misc{Shing2025TAIDTA,
      title={TAID: Temporally Adaptive Interpolated Distillation for Efficient Knowledge Transfer in Language Models}, 
      author={Makoto Shing and Kou Misaki and Han Bao and Sho Yokoi and Takuya Akiba},
      year={2025},
      eprint={2501.16937},
      archivePrefix={arXiv},
      primaryClass={cs.LG},
      url={https://arxiv.org/abs/2501.16937}, 
}

@misc{Yang2025Qwen3TR,
      title={Qwen3 Technical Report}, 
      author={An Yang and Anfeng Li and Baosong Yang and Beichen Zhang and Binyuan Hui and Bo Zheng and Bowen Yu and Chang Gao and Chengen Huang and Chenxu Lv and Chujie Zheng and Dayiheng Liu and Fan Zhou and Fei Huang and Feng Hu and Hao Ge and Haoran Wei and Huan Lin and Jialong Tang and Jian Yang and Jianhong Tu and Jianwei Zhang and Jianxin Yang and Jiaxi Yang and Jing Zhou and Jingren Zhou and Junyang Lin and Kai Dang and Keqin Bao and Kexin Yang and Le Yu and Lianghao Deng and Mei Li and Mingfeng Xue and Mingze Li and Pei Zhang and Peng Wang and Qin Zhu and Rui Men and Ruize Gao and Shixuan Liu and Shuang Luo and Tianhao Li and Tianyi Tang and Wenbiao Yin and Xingzhang Ren and Xinyu Wang and Xinyu Zhang and Xuancheng Ren and Yang Fan and Yang Su and Yichang Zhang and Yinger Zhang and Yu Wan and Yuqiong Liu and Zekun Wang and Zeyu Cui and Zhenru Zhang and Zhipeng Zhou and Zihan Qiu},
      year={2025},
      eprint={2505.09388},
      archivePrefix={arXiv},
      primaryClass={cs.CL},
      url={https://arxiv.org/abs/2505.09388}, 
}

@misc{Zeng2026GLM5FV,
      title={GLM-5: from Vibe Coding to Agentic Engineering}, 
      author={GLM-5-Team and : and Aohan Zeng and Xin Lv and Zhenyu Hou and Zhengxiao Du and Qinkai Zheng and Bin Chen and Da Yin and Chendi Ge and Chenghua Huang and Chengxing Xie and Chenzheng Zhu and Congfeng Yin and Cunxiang Wang and Gengzheng Pan and Hao Zeng and Haoke Zhang and Haoran Wang and Huilong Chen and Jiajie Zhang and Jian Jiao and Jiaqi Guo and Jingsen Wang and Jingzhao Du and Jinzhu Wu and Kedong Wang and Lei Li and Lin Fan and Lucen Zhong and Mingdao Liu and Mingming Zhao and Pengfan Du and Qian Dong and Rui Lu and Shuang-Li and Shulin Cao and Song Liu and Ting Jiang and Xiaodong Chen and Xiaohan Zhang and Xuancheng Huang and Xuezhen Dong and Yabo Xu and Yao Wei and Yifan An and Yilin Niu and Yitong Zhu and Yuanhao Wen and Yukuo Cen and Yushi Bai and Zhongpei Qiao and Zihan Wang and Zikang Wang and Zilin Zhu and Ziqiang Liu and Zixuan Li and Bojie Wang and Bosi Wen and Can Huang and Changpeng Cai and Chao Yu and Chen Li and Chengwei Hu and Chenhui Zhang and Dan Zhang and Daoyan Lin and Dayong Yang and Di Wang and Ding Ai and Erle Zhu and Fangzhou Yi and Feiyu Chen and Guohong Wen and Hailong Sun and Haisha Zhao and Haiyi Hu and Hanchen Zhang and Hanrui Liu and Hanyu Zhang and Hao Peng and Hao Tai and Haobo Zhang and He Liu and Hongwei Wang and Hongxi Yan and Hongyu Ge and Huan Liu and Huanpeng Chu and Jia'ni Zhao and Jiachen Wang and Jiajing Zhao and Jiamin Ren and Jiapeng Wang and Jiaxin Zhang and Jiayi Gui and Jiayue Zhao and Jijie Li and Jing An and Jing Li and Jingwei Yuan and Jinhua Du and Jinxin Liu and Junkai Zhi and Junwen Duan and Kaiyue Zhou and Kangjian Wei and Ke Wang and Keyun Luo and Laiqiang Zhang and Leigang Sha and Liang Xu and Lindong Wu and Lintao Ding and Lu Chen and Minghao Li and Nianyi Lin and Pan Ta and Qiang Zou and Rongjun Song and Ruiqi Yang and Shangqing Tu and Shangtong Yang and Shaoxiang Wu and Shengyan Zhang and Shijie Li and Shuang Li and Shuyi Fan and Wei Qin and Wei Tian and Weining Zhang and Wenbo Yu and Wenjie Liang and Xiang Kuang and Xiangmeng Cheng and Xiangyang Li and Xiaoquan Yan and Xiaowei Hu and Xiaoying Ling and Xing Fan and Xingye Xia and Xinyuan Zhang and Xinze Zhang and Xirui Pan and Xu Zou and Xunkai Zhang and Yadi Liu and Yandong Wu and Yanfu Li and Yidong Wang and Yifan Zhu and Yijun Tan and Yilin Zhou and Yiming Pan and Ying Zhang and Yinpei Su and Yipeng Geng and Yong Yan and Yonglin Tan and Yuean Bi and Yuhan Shen and Yuhao Yang and Yujiang Li and Yunan Liu and Yunqing Wang and Yuntao Li and Yurong Wu and Yutao Zhang and Yuxi Duan and Yuxuan Zhang and Zezhen Liu and Zhengtao Jiang and Zhenhe Yan and Zheyu Zhang and Zhixiang Wei and Zhuo Chen and Zhuoer Feng and Zijun Yao and Ziwei Chai and Ziyuan Wang and Zuzhou Zhang and Bin Xu and Minlie Huang and Hongning Wang and Juanzi Li and Yuxiao Dong and Jie Tang},
      year={2026},
      eprint={2602.15763},
      archivePrefix={arXiv},
      primaryClass={cs.LG},
      url={https://arxiv.org/abs/2602.15763}, 
}

@misc{DeepSeekAI2026DeepSeekV4TH,
      title={DeepSeek-V4: Towards Highly Efficient Million-Token Context Intelligence}, 
      author={DeepSeek-AI and Anyi Xu and Bangcai Lin and Bing Xue and Bingxuan Wang and Bingzheng Xu and Bochao Wu and Bowei Zhang and Chaofan Lin and Chen Dong and Chenchen Ling and Chengda Lu and Chenggang Zhao and Chengqi Deng and Chengyu Hou and Chenhao Xu and Chenze Shao and Chong Ruan and Conner Sun and Damai Dai and Daya Guo and Dejian Yang and Deli Chen and Donghao Li and Dongjie Ji and Erhang Li and Fang Wei and Fangyun Lin and Fangzhou Yuan and Feiyu Xia and Fucong Dai and Guangbo Hao and Guanting Chen and Guoai Cao and Guolai Meng and Guowei Li and Han Yu and Han Zhang and Hanwei Xu and Hao Li and Haofen Liang and Haoling Zhang and Haoming Luo and Haoran Wei and Haotian Yuan and Haowei Zhang and Haowen Luo and Haoyu Chen and Haozhe Ji and Hengqing Zhang and Honghui Ding and Hongxuan Tang and Huanqi Cao and Huazuo Gao and Hui Qu and Hui Zeng and J Yang and JQ Zhu and Jia Luo and Jia Song and Jia Yu and Jialiang Huang and Jialu Cai and Jian Liang and Jiangting Zhou and Jiasheng Ye and Jiashi Li and Jiaxin Xu and Jiewen Hu and Jieyu Yang and Jin Chen and Jin Yan and Jingchang Chen and Jingli Zhou and Jingting Xiang and Jingyang Yuan and Jingyuan Cheng and Jingzi Zhou and Jinhua Zhu and Jiping Yu and Joseph Sun and Jun Ran and Junguang Jiang and Junjie Qiu and Junlong Li and Junmin Zheng and Junxiao Song and Kai Dong and Kaige Gao and Kang Guan and Kexing Zhou and Kezhao Huang and Kuai Yu and Lean Wang and Lecong Zhang and Lei Wang and Leyi Xia and Li Zhang and Liang Zhao and Lihua Guo and Lingxiao Luo and Linwang Ma and Linyan Zhu and Litong Wang and Liyu Cai and Liyue Zhang and Longhao Chen and MS Di and MY Xu and Max Mei and Miaojun Wang and Mingchuan Zhang and Minghua Zhang and Minghui Tang and Mingming Li and Mingxu Zhou and Minmin Han and Ning Wang and Panpan Huang and Panpan Wang and Peixin Cong and Peiyi Wang and Peng Zhang and Qiancheng Wang and Qihao Zhu and Qingyang Li and Qinyu Chen and Qiushi Du and Qiwei Jiang and Rui Tian and Ruifan Xu and Ruijie Lu and Ruiling Xu and Ruiqi Ge and Ruisong Zhang and Ruizhe Pan and Runji Wang and Runqian Chen and Runqiu Yin and Runxin Xu and Ruomeng Shen and Ruoyu Zhang and Ruyi Chen and SH Liu and Shanghao Lu and Shangmian Sun and Shangyan Zhou and Shanhuang Chen and Shaofei Cai and Shaoheng Nie and Shaoqing Wu and Shaoyuan Chen and Shengding Hu and Shengyu Liu and Shiqiang Hu and Shirong Ma and Shiyu Wang and Shuiping Yu and Shunfeng Zhou and Shuting Pan and Shuying Yu and Songyang Zhou and Tao Ni and Tao Yun and Tian Jin and Tian Pei and Tian Ye and Tianle Lin and Tianran Ji and Tianyi Cui and Tianyuan Yue and Tingting Yu and Tun Wang and W Zhang and WL Xiao and Wangding Zeng and Wei An and Weilin Zhao and Wen Liu and Wenfeng Liang and Wenjie Pang and Wenjing Luo and Wenjing Yao and Wenjun Gao and Wenkai Yang and Wenlve Huang and Wenqing Hou and Wentao Zhang and Wenting Ma and Xi Gao and Xiang He and Xiangwen Wang and Xianzu Wang and Xiao Bi and Xiaodong Liu and Xiaohan Wang and Xiaokang Chen and Xiaokang Zhang and Xiaotao Nie and Xiaowen Sun and Xiaoxiang Wang and Xin Cheng and Xin Liu and Xin Xie and Xingchao Liu and Xingchen Liu and Xingkai Yu and Xingyou Li and Xinyu Yang and Xinyu Zhang and Xu Chen and Xuanyu Wang and Xuecheng Su and Xueyin Chen and Xuheng Lin and Xuwei Fu and YC Yan and YQ Wang and YW Ma and Yanfeng Luo and Yang Zhang and Yanhong Xu and Yanru Ma and Yanwen Huang and Yao Li and Yao Li and Yao Xu and Yao Zhao and Yaofeng Sun and Yaohui Wang and Yi Qian and Yi Shao and Yi Yu and Yichao Zhang and Yifan Ding and Yifan Shi and Yijia Wu and Yiliang Xiong and Yiling Ma and Ying He and Ying Tang and Ying Zhou and Yingjia Luo and Yinmin Zhong and Yishi Piao and Yisong Wang and Yixiang Zhang and Yixiao Chen and Yixuan Tan and Yixuan Wei and Yiyang Ma and Yiyuan Liu and Yonglun Yang and Yongqiang Guo and Yongtong Wu and Yu Wu and YuKun Li and Yuan Cheng and Yuan Ou and Yuanfan Xu and Yuanhao Li and Yuduan Wang and Yuehan Yang and Yuer Xu and Yuhan Wu and Yuhao Meng and Yuheng Zou and Yukun Zha and Yunfan Xiong and Yupeng Chen and Yuping Lin and Yuqian Cao and Yuqian Wang and Yushun Zhang and Yuting Yan and Yutong Lin and Yuxian Gu and Yuxiang Luo and Yuxiang You and Yuxuan Liu and Yuxuan Zhou and Yuyang Zhou and Yuzhen Huang and ZF Wu and Zehao Wang and Zehua Zhao and Zehui Ren and Zekai Zhang and Zhangli Sha and Zhe Fu and Zhe Ju and Zhean Xu and Zhenda Xie and Zhengyan Zhang and Zheren Gao and Zhewen Hao and Zhibin Gou and Zhicheng Ma and Zhigang Yan and Zhihong Shao and Zhixian Huang and Zhixuan Chen and Zhiyu Wu and Zhizhou Ren and Zhongyu Wu and Zhuoshu Li and Zhuping Zhang and Zian Xu and Zihao Wang and Zihua Qu and Zihui Gu and Zijia Zhu and Zilin Li and Zipeng Zhang and Ziwei Xie and Ziyi Gao and Ziyi Wan and Zizheng Pan and Zongqing Yao},
      year={2026},
      eprint={2606.19348},
      archivePrefix={arXiv},
      primaryClass={cs.CL},
      url={https://arxiv.org/abs/2606.19348}, 
}

@inproceedings{Sheng2024HybridFlowAF, series={EuroSys ’25},
   title={HybridFlow: A Flexible and Efficient RLHF Framework},
   url={http://dx.doi.org/10.1145/3689031.3696075},
   DOI={10.1145/3689031.3696075},
   booktitle={Proceedings of the Twentieth European Conference on Computer Systems},
   publisher={ACM},
   author={Sheng, Guangming and Zhang, Chi and Ye, Zilingfeng and Wu, Xibin and Zhang, Wang and Zhang, Ru and Peng, Yanghua and Lin, Haibin and Wu, Chuan},
   year={2025},
   month=Mar, pages={1279–1297},
   collection={EuroSys ’25} }

@misc{Kwon2023EfficientMM,
      title={Efficient Memory Management for Large Language Model Serving with PagedAttention}, 
      author={Woosuk Kwon and Zhuohan Li and Siyuan Zhuang and Ying Sheng and Lianmin Zheng and Cody Hao Yu and Joseph E. Gonzalez and Hao Zhang and Ion Stoica},
      year={2023},
      eprint={2309.06180},
      archivePrefix={arXiv},
      primaryClass={cs.LG},
      url={https://arxiv.org/abs/2309.06180}, 
}

@misc{olmo2026olmo3,
      title={Olmo 3}, 
      author={Team Olmo and : and Allyson Ettinger and Amanda Bertsch and Bailey Kuehl and David Graham and David Heineman and Dirk Groeneveld and Faeze Brahman and Finbarr Timbers and Hamish Ivison and Jacob Morrison and Jake Poznanski and Kyle Lo and Luca Soldaini and Matt Jordan and Mayee Chen and Michael Noukhovitch and Nathan Lambert and Pete Walsh and Pradeep Dasigi and Robert Berry and Saumya Malik and Saurabh Shah and Scott Geng and Shane Arora and Shashank Gupta and Taira Anderson and Teng Xiao and Tyler Murray and Tyler Romero and Victoria Graf and Akari Asai and Akshita Bhagia and Alexander Wettig and Alisa Liu and Aman Rangapur and Chloe Anastasiades and Costa Huang and Dustin Schwenk and Harsh Trivedi and Ian Magnusson and Jaron Lochner and Jiacheng Liu and Lester James V. Miranda and Maarten Sap and Malia Morgan and Michael Schmitz and Michal Guerquin and Michael Wilson and Regan Huff and Ronan Le Bras and Rui Xin and Rulin Shao and Sam Skjonsberg and Shannon Zejiang Shen and Shuyue Stella Li and Tucker Wilde and Valentina Pyatkin and Will Merrill and Yapei Chang and Yuling Gu and Zhiyuan Zeng and Ashish Sabharwal and Luke Zettlemoyer and Pang Wei Koh and Ali Farhadi and Noah A. Smith and Hannaneh Hajishirzi},
      year={2026},
      eprint={2512.13961},
      archivePrefix={arXiv},
      primaryClass={cs.CL},
      url={https://arxiv.org/abs/2512.13961}, 
}

@article{He2025DeepMath103KAL,
  title={DeepMath-103K: A Large-Scale, Challenging, Decontaminated, and Verifiable Mathematical Dataset for Advancing Reasoning},
  author={Zhiwei He and Tian Liang and Jiahao Xu and Qiuzhi Liu and Xingyu Chen and Yue Wang and Linfeng Song and Dian Yu and Zhenwen Liang and Wenxuan Wang and Zhuosheng Zhang and Rui Wang and Zhaopeng Tu and Haitao Mi and Dong Yu},
  journal={ArXiv},
  year={2025},
  volume={abs/2504.11456},
  url={https://api.semanticscholar.org/CorpusID:277787455}
}

@misc{amc23,
  author = {{math-ai}},
  title  = {{AMC} 2023},
  year   = {2023},
  url    = {https://huggingface.co/datasets/math-ai/amc23},
  note   = {Dataset compiled from the 2023 American Mathematics Competitions}
}

@misc{hmmt2025feb,
  author = {{Harvard--MIT Mathematics Tournament}},
  title  = {{HMMT} {February} 2025: Problems and Solutions},
  year   = {2025},
  url    = {https://www.hmmt.org/www/archive/282}
}

@misc{aime24,
      title={American Invitational Mathematics Examination (AIME) 2024}, 
      author={Zhang, Yifan and Math-AI, Team},
      year={2024},
}

@misc{aime25,
      title={American Invitational Mathematics Examination (AIME) 2025}, 
      author={Zhang, Yifan and Math-AI, Team},
      year={2025},
}

@article{Chen2021EvaluatingLL,
  title={Evaluating Large Language Models Trained on Code},
  author={Mark Chen and Jerry Tworek and Heewoo Jun and Qiming Yuan and Henrique Pond{\'e} de Oliveira Pinto and Jared Kaplan and Harrison Edwards and Yuri Burda and Nicholas Joseph and Greg Brockman and Alex Ray and Raul Puri and Gretchen Krueger and Michael Petrov and Heidy Khlaaf and Girish Sastry and Pamela Mishkin and Brooke Chan and Scott Gray and Nick Ryder and Mikhail Pavlov and Alethea Power and Lukasz Kaiser and Mo Bavarian and Clemens S. Winter and Phil Tillet and Felipe Petroski Such and David W. Cummings and Matthias Plappert and Fotios Chantzis and Elizabeth Barnes and Ariel Herbert-Voss and William Hebgen Guss and Alex Nichol and Alex Paino and Nikolas Tezak and Jie Tang and Igor Babuschkin and Suchir Balaji and Shantanu Jain and William Saunders and Christopher Hesse and Andrew Carr and Jan Leike and Joshua Achiam and Vedant Misra and Evan Morikawa and Alec Radford and Matthew M. Knight and Miles Brundage and Mira Murati and Katie Mayer and Peter Welinder and Bob McGrew and Dario Amodei and Sam McCandlish and Ilya Sutskever and Wojciech Zaremba},
  journal={ArXiv},
  year={2021},
  volume={abs/2107.03374},
  url={https://api.semanticscholar.org/CorpusID:235755472}
}

@article{Zhang2025Qwen3EA,
  title={Qwen3 Embedding: Advancing Text Embedding and Reranking Through Foundation Models},
  author={Yanzhao Zhang and Mingxin Li and Dingkun Long and Xin Zhang and Huan Lin and Baosong Yang and Pengjun Xie and An Yang and Dayiheng Liu and Junyang Lin and Fei Huang and Jingren Zhou},
  journal={ArXiv},
  year={2025},
  volume={abs/2506.05176},
  url={https://api.semanticscholar.org/CorpusID:279243736}
}

@article{Lin2025ZebraLogicOT,
  title={ZebraLogic: On the Scaling Limits of LLMs for Logical Reasoning},
  author={Bill Yuchen Lin and Ronan Le Bras and Kyle Richardson and Ashish Sabharwal and Radha Poovendran and Peter Clark and Yejin Choi},
  journal={ArXiv},
  year={2025},
  volume={abs/2502.01100},
  url={https://api.semanticscholar.org/CorpusID:276094621}
}

@article{Hao2025RethinkingEI,
  title={Rethinking Entropy Interventions in RLVR: An Entropy Change Perspective},
  author={Zhezheng Hao and Hong Wang and Haoyang Liu and Jian Luo and Jiarui Yu and Hande Dong and Qiang Lin and Can Wang and Jiawei Chen},
  journal={ArXiv},
  year={2025},
  volume={abs/2510.10150},
  url={https://api.semanticscholar.org/CorpusID:282057600}
}

@misc{li2026rethinkingonpolicydistillationlarge,
      title={Rethinking On-Policy Distillation of Large Language Models: Phenomenology, Mechanism, and Recipe}, 
      author={Yaxuan Li and Yuxin Zuo and Bingxiang He and Jinqian Zhang and Chaojun Xiao and Cheng Qian and Tianyu Yu and Huan-ang Gao and Wenkai Yang and Zhiyuan Liu and Ning Ding},
      year={2026},
      eprint={2604.13016},
      archivePrefix={arXiv},
      primaryClass={cs.LG},
      url={https://arxiv.org/abs/2604.13016}, 
}
\bibliographystyle{preprint}

\appendix

\clearpage
\section{Experimental Details}
\label{app:implementation}

This section expands the experimental protocol in Section~\ref{sec:exp-setup}.
Complete results and standard deviations are given in Appendix~\ref{app:full-results}.

\subsection{Data and Model Configurations}
\label{app:data-config}

Training prompts come from DeepMath-103K \citep{He2025DeepMath103KAL}, filtered to the difficulty
bands in Table~\ref{tab:app-data-config}. Within each configuration, the same reasoning mode
is used for the student, teacher (where applicable), and evaluation. The Instruct configurations use
Qwen3-4B-Instruct-2507 and Olmo-3-7B-Instruct.

\begin{table}[htbp]
\centering
\caption{Backbone configurations and DeepMath training difficulty bands.}
\label{tab:app-data-config}
\normalsize
\setlength{\tabcolsep}{7pt}
\begin{tabular}{lll}
\toprule
Backbone & Reasoning mode & Difficulty levels \\
\midrule
Qwen3-0.6B & Thinking & 1--3 \\
Qwen3-1.7B & Thinking & 5--7 \\
Qwen3-4B-Instruct-2507 & Instruct & 6--8 \\
Olmo-3-7B-Instruct & Instruct & 7--10 \\
\bottomrule
\end{tabular}
\end{table}

For in-domain mathematical reasoning, we evaluate on AMC23, AIME'24, AIME'25,
and HMMT'25 Feb.
For out-of-domain evaluation, we use ZebraLogic for logical reasoning and
HumanEval for code generation. We construct a fixed subset of 100 ZebraLogic
puzzles by sampling four puzzles per grid size, using the same subset for all
configurations.

\subsection{Shared Training Parameters and Budgets}
\label{app:shared-training}
All methods and ablations use veRL with vLLM rollouts on a single 8$\times$H200 node. They share the training settings in Table~\ref{tab:app-training-config} and a budget of 360 optimizer steps.

\begin{table}[htbp]
\centering
\caption{Shared training parameters.}
\label{tab:app-training-config}
\normalsize
\setlength{\tabcolsep}{9pt}
\begin{tabular}{ll}
\toprule
Parameter & Value \\
\midrule
Optimizer & AdamW \\
Learning rate & $1\times10^{-6}$ \\
Warmup & 100 steps \\
Weight decay & $0.01$ \\
Gradient-norm clipping & $0.1$ \\
\midrule
Prompts per rollout round & 256 \\
Responses per prompt & $G=8$ \\
Mini-batch size & 32 prompt groups \\
Optimizer updates per round & 8 \\
\midrule
Rollout temperature & $1.0$ \\
Rollout top-$p$ / top-$k$ & $0.95$ / 20 \\
Maximum rollout response length & 16,384 tokens \\
\bottomrule
\end{tabular}
\end{table}

\subsection{Evaluation Protocol}
\label{app:eval-protocol}

We evaluate models with the settings in Table~\ref{tab:app-eval-config}. We report mean@16 averaged over three evaluation seeds.
Evaluation uses the student alone
with the original question and the backbone's corresponding reasoning mode.

\begin{table}[htbp]
\centering
\caption{Evaluation parameters. Batch sizes count prompts.}
\label{tab:app-eval-config}
\normalsize
\setlength{\tabcolsep}{7pt}
\begin{tabular}{lll}
\toprule
Group & Parameter & Value \\
\midrule
\multirow{5}{*}{Evaluation}
 & Responses per problem & 16 \\
 & Evaluation seeds & 17, 42, 1234 \\
 & Temperature & $0.6$ \\
 & Top-$p$ / top-$k$ & $0.95$ / 20 \\
 & Maximum prompt length & 2,048 tokens \\
\midrule
\multirow{3}{*}{Maximum response length}
 & Math reasoning & 38,912 tokens \\
 & Logic (ZebraLogic) & 38,912 tokens \\
 & Coding (HumanEval) & 16,384 tokens \\
\midrule
\multirow{3}{*}{Evaluation batch size}
 & Math reasoning & 4 prompts \\
 & Logic (ZebraLogic) & 16 prompts \\
 & Coding (HumanEval) & 8 prompts \\
\bottomrule
\end{tabular}
\end{table}

\subsection{Method-Specific Implementation and Exemplars}
\label{app:method-implementation}

Table~\ref{tab:app-method-config} summarizes the settings that differ across methods.
OPSD + EOPD adapts EOPD~\citep{jin2026entropy} by using OPSD's reference-conditioned
self-teacher. Distillation uses gradient-free teacher passes; EMA and snapshot updates
follow the schedules below. Entropy is computed over the full vocabulary.

\begingroup
\normalsize
\setlength{\tabcolsep}{4pt}
\renewcommand{\arraystretch}{1.08}
\begin{longtable}{>{\raggedright\arraybackslash}p{0.16\textwidth}>{\raggedright\arraybackslash}p{0.22\textwidth}>{\raggedright\arraybackslash}p{0.21\textwidth}>{\raggedright\arraybackslash}p{\dimexpr0.41\textwidth-8\tabcolsep\relax}}
\caption{Method-specific objectives and training settings.}
\label{tab:app-method-config}\\
\toprule
Method & Setting & Value & Description \\
\midrule
\endfirsthead
\multicolumn{4}{l}{\normalsize Table~\thetable{} (continued)}\\
\toprule
Method & Setting & Value & Description \\
\midrule
\endhead
\multicolumn{4}{r}{\normalsize Continued on next page}\\
\endfoot
\bottomrule
\endlastfoot
OPSD & Teacher & Frozen initial weights & Current reference solution as context. \\*
 & Objective & Forward KL & Teacher-to-student distillation. \\*
 & Support / clipping & Top-100 + tail; $0.05$ & Teacher-selected support; per-element cap. \\
\midrule
SDPO & Teacher / EMA decay & EMA; $0.95$ & Updated after each optimizer step. \\*
 & Teacher context & Successful sibling & Excludes the rollout being scored. \\*
 & JSD mixing $\alpha$ & $0.5$ & Equal student--teacher mixture. \\*
 & Support & Top-100 + tail & Student-selected support. \\*
 & Importance-ratio cap & $2.0$ & Upper cap on the distillation ratio. \\
\midrule
GRPO & Reward / advantage & Binary; group-normalized & Verifiable correctness; normalized within each response group. \\*
 & Policy-ratio clip & $\epsilon=0.2$ & Clipped policy-gradient objective. \\*
 & Reference / KL weight & Frozen initial weights; $0.001$ & KL regularization in the actor loss. \\
\midrule
RLSD & Teacher / refresh & Snapshot; every 10 steps & Copies student weights every 10 optimizer steps. \\*
 & PPO / evidence clips & $\epsilon=0.2$; $\epsilon_w=0.2$ & Policy-ratio and evidence-weight clipping. \\*
 & Modulation $\lambda$ & $0.5\rightarrow0$ & Linear decay over 50 training steps. \\
\midrule
OPSD + EOPD & Teacher & Frozen initial weights & Current reference solution as context. \\*
 & Entropy gate & $H_T\geq0.8$ nats & Adds forward KL to sampled reverse KL. \\*
 & Forward-KL coefficient & $1.0$ & Weight of the entropy-gated term. \\*
 & Forward-KL support & Top-16; no tail & Renormalized teacher, full-vocabulary student; no per-element cap. \\
\midrule
\projectname{} & Teacher & Frozen initial weights & One retrieved solved neighbor as context. \\*
 & Objective & BER + AGD & Entropy gap sets direction and weight. \\*
 & Support / clipping & Top-100 + tail; $0.05$ & Teacher-selected support; per-element cap. \\*
 & Exemplar candidates & 15 neighbors & Self-excluded; one draw per rollout. \\
\end{longtable}
\endgroup

\paragraph{Exemplar construction and teacher prompt.}
Using Qwen3-Embedding-8B~\citep{Zhang2025Qwen3EA}, we build a self-excluded
15-neighbor index within each difficulty band. Each rollout draws one neighbor and
its dataset reference solution, with sampling seeded by the training seed, step,
problem, and rollout slot. The box reproduces the fixed user message with variable
fields in braces. After applying the backbone's native chat template and reasoning
mode, the teacher scores the student's continuation conditioned on
\texttt{\{student\_prefix\}}. Exemplars are used only during training.
For Qwen3-1.7B, one-time embedding and neighbor-index construction took 219 seconds
on eight H200 GPUs, accounting for 0.41\% of the GPU time spent on rollout
generation and actor updates over 360 optimizer steps.

\begin{tcolorbox}[
  breakable,
  colback=gray!3,colframe=gray!55,boxrule=0.5pt,arc=1mm,
  title={EGT teacher prompt},fonttitle=\bfseries,
  fontupper=\normalsize\ttfamily\raggedright,left=7pt,right=7pt,top=6pt,bottom=6pt,
  before skip=6pt,after skip=8pt]
Problem: \{current\_problem\}

\medskip
Here is a reference solution to a DIFFERENT but related problem. It is not the
problem you must solve, and its final answer is NOT the answer to your problem
-- use it only as a guide to the strategies and techniques that transfer:

\medskip
=== Reference Solution Begin ===\par
Related problem: \{related\_problem\}

\medskip
Solution to the related problem:\par
\{related\_solution\}\par
=== Reference Solution End ===

\medskip
After reading the reference solution above, make sure you understand the
reasoning behind each step, then solve the original problem with whatever
transfers from it.

\medskip
Please reason step by step, and put your final answer within \textbackslash boxed\{\}.
\tcblower
\normalfont\normalsize
\textbf{Assistant continuation:} \texttt{\{student\_prefix\}}\par
The teacher scores the next token after this prefix; this line is not part of the user message.
\end{tcolorbox}

\clearpage
\section{Full Evaluation Results}
\label{app:full-results}

\subsection{Evaluation Metrics}
\label{app:eval-metrics}
For $N$ problems, let $c_i(m)$ be the number of correct answers among $m$ independently
sampled responses to problem $i$. We compute
\[
\mathrm{mean}@k = \frac{1}{N}\sum_{i=1}^{N}\frac{c_i(k)}{k},
\qquad
\widehat{\mathrm{pass}@k} = \frac{1}{N}\sum_{i=1}^{N}
\left[1-\frac{\binom{n-c_i(n)}{k}}{\binom{n}{k}}\right].
\]
Here $k\le n=16$, with $\binom{a}{k}=0$ for $a<k$. Estimated pass@1 equals mean@16,
so its duplicate table is omitted. Entries are mean ${\pm}$ population standard deviation (\%)
over three seeds; bold/underlined means indicate best/second-best distinct results.

\subsection{In-domain Mathematical Reasoning}
\label{app:in-domain-results}
\begin{table}[H]
\caption{Mathematical reasoning mean@16 (\%). Avg. equally weights the four benchmarks.}
\label{tab:app-mean16}
\centering
\normalsize
\setlength{\tabcolsep}{0.25pt}
\renewcommand{\arraystretch}{0.96}
\setlength{\aboverulesep}{0.3ex}
\setlength{\belowrulesep}{0.35ex}
\begin{tabular}{p{\dimexpr0.14\textwidth-2\tabcolsep\relax}>{\raggedright\arraybackslash}p{\dimexpr0.20\textwidth-2\tabcolsep\relax}*{5}{>{\centering\arraybackslash}p{\dimexpr0.132\textwidth-2\tabcolsep\relax}}}
\toprule
 & & \multicolumn{5}{c}{\textbf{Math Reasoning} (mean@16)} \\
\cmidrule(lr){3-7}
\textbf{Model} & \textbf{Method} & AMC23 & AIME24 & AIME25 & HMMT25 & Avg. \\
\midrule
\multirow{7}{=}{\centering Qwen3-0.6B\\(Thinking)} & Base & $45.81{\pm}0.98$ & $9.79{\pm}0.29$ & $14.17{\pm}1.12$ & $8.26{\pm}0.43$ & $19.51{\pm}0.19$ \\
 & OPSD & $45.88{\pm}0.86$ & $10.90{\pm}0.86$ & $14.17{\pm}0.59$ & $\mathbf{9.17}{\pm}0.51$ & $20.03{\pm}0.49$ \\
 & SDPO & $45.48{\pm}0.27$ & $9.38{\pm}0.34$ & $14.58{\pm}1.39$ & $7.78{\pm}1.70$ & $19.30{\pm}0.62$ \\
 & GRPO & $46.11{\pm}0.98$ & $9.79{\pm}0.17$ & $\mathbf{16.39}{\pm}0.26$ & $8.40{\pm}0.77$ & $20.17{\pm}0.03$ \\
 & RLSD & $\mathbf{46.84}{\pm}1.95$ & $\underline{10.97}{\pm}0.49$ & $15.14{\pm}0.64$ & $7.78{\pm}0.26$ & $\underline{20.18}{\pm}0.53$ \\
 & \mbox{OPSD + EOPD} & $46.36{\pm}0.62$ & $9.79{\pm}0.95$ & $15.90{\pm}0.35$ & $8.54{\pm}1.03$ & $20.15{\pm}0.24$ \\
\rowcolor[gray]{0.85}  & \textbf{\mbox{\projectname{} (Ours)}} & $\underline{46.56}{\pm}1.37$ & $\mathbf{11.04}{\pm}0.29$ & $\underline{16.18}{\pm}0.35$ & $\underline{8.89}{\pm}1.39$ & $\mathbf{20.67}{\pm}0.42$ \\
\midrule
\multirow{7}{=}{\centering Qwen3-1.7B\\(Thinking)} & Base & $75.50{\pm}0.28$ & $46.04{\pm}1.78$ & $34.58{\pm}0.29$ & $21.88{\pm}0.78$ & $44.50{\pm}0.27$ \\
 & OPSD & $74.87{\pm}0.46$ & $45.97{\pm}1.58$ & $35.00{\pm}1.36$ & $\underline{22.71}{\pm}0.61$ & $44.64{\pm}0.41$ \\
 & SDPO & $\underline{76.05}{\pm}0.98$ & $47.99{\pm}2.04$ & $\underline{35.28}{\pm}1.72$ & $21.88{\pm}1.28$ & $\underline{45.30}{\pm}0.29$ \\
 & GRPO & $74.65{\pm}0.09$ & $47.01{\pm}1.39$ & $35.21{\pm}0.95$ & $21.88{\pm}0.90$ & $44.69{\pm}0.42$ \\
 & RLSD & $75.85{\pm}0.20$ & $\underline{48.26}{\pm}0.10$ & $33.96{\pm}1.39$ & $21.81{\pm}0.55$ & $44.97{\pm}0.52$ \\
 & \mbox{OPSD + EOPD} & $74.27{\pm}0.15$ & $47.43{\pm}1.28$ & $34.24{\pm}1.21$ & $21.46{\pm}0.78$ & $44.35{\pm}0.09$ \\
\rowcolor[gray]{0.85}  & \textbf{\mbox{\projectname{} (Ours)}} & $\mathbf{79.02}{\pm}0.41$ & $\mathbf{52.29}{\pm}0.45$ & $\mathbf{39.24}{\pm}0.26$ & $\mathbf{25.00}{\pm}0.90$ & $\mathbf{48.89}{\pm}0.32$ \\
\midrule
\multirow{7}{=}{\centering Qwen3-4B-\\Instruct-2507} & Base & $85.39{\pm}0.48$ & $61.74{\pm}0.49$ & $46.46{\pm}2.04$ & $29.03{\pm}0.71$ & $55.65{\pm}0.79$ \\
 & OPSD & $84.81{\pm}0.09$ & $\underline{62.36}{\pm}1.39$ & $45.63{\pm}1.47$ & $\underline{29.51}{\pm}0.84$ & $55.58{\pm}0.52$ \\
 & SDPO & $83.86{\pm}0.69$ & $61.32{\pm}1.32$ & $45.83{\pm}1.45$ & $27.50{\pm}1.06$ & $54.63{\pm}0.66$ \\
 & GRPO & $84.44{\pm}0.77$ & $61.67{\pm}0.45$ & $45.56{\pm}1.83$ & $29.10{\pm}1.70$ & $55.19{\pm}0.55$ \\
 & RLSD & $85.12{\pm}0.84$ & $62.22{\pm}0.94$ & $\underline{46.94}{\pm}0.84$ & $28.40{\pm}0.26$ & $\underline{55.67}{\pm}0.30$ \\
 & \mbox{OPSD + EOPD} & $\underline{85.47}{\pm}0.49$ & $62.01{\pm}2.21$ & $46.46{\pm}0.29$ & $28.40{\pm}0.26$ & $55.59{\pm}0.71$ \\
\rowcolor[gray]{0.85}  & \textbf{\mbox{\projectname{} (Ours)}} & $\mathbf{87.10}{\pm}0.63$ & $\mathbf{62.85}{\pm}1.09$ & $\mathbf{50.97}{\pm}1.13$ & $\mathbf{30.21}{\pm}0.61$ & $\mathbf{57.78}{\pm}0.34$ \\
\midrule
\multirow{7}{=}{\centering Olmo-3-7B-\\Instruct} & Base & $80.52{\pm}0.29$ & $52.85{\pm}1.21$ & $37.78{\pm}0.20$ & $23.26{\pm}0.35$ & $48.60{\pm}0.39$ \\
 & OPSD & $80.52{\pm}0.40$ & $52.78{\pm}1.32$ & $38.26{\pm}1.79$ & $22.36{\pm}0.49$ & $48.48{\pm}0.69$ \\
 & SDPO & $80.27{\pm}0.12$ & $51.32{\pm}0.79$ & $\underline{39.93}{\pm}1.74$ & $\mathbf{24.51}{\pm}1.43$ & $49.01{\pm}0.14$ \\
 & GRPO & $\underline{80.92}{\pm}1.08$ & $52.64{\pm}0.39$ & $38.61{\pm}0.80$ & $23.47{\pm}0.52$ & $48.91{\pm}0.18$ \\
 & RLSD & $80.37{\pm}0.51$ & $\underline{53.61}{\pm}1.09$ & $38.82{\pm}0.10$ & $23.82{\pm}0.26$ & $\underline{49.16}{\pm}0.31$ \\
 & \mbox{OPSD + EOPD} & $79.62{\pm}0.47$ & $53.47{\pm}1.19$ & $39.17{\pm}2.17$ & $\underline{23.96}{\pm}1.47$ & $49.05{\pm}0.37$ \\
\rowcolor[gray]{0.85}  & \textbf{\mbox{\projectname{} (Ours)}} & $\mathbf{81.22}{\pm}1.05$ & $\mathbf{56.25}{\pm}1.35$ & $\mathbf{40.00}{\pm}2.13$ & $23.89{\pm}0.35$ & $\mathbf{50.34}{\pm}0.75$ \\
\bottomrule
\end{tabular}
\end{table}

\begin{table}[H]
\caption{Mathematical reasoning unbiased pass@2 (\%).}
\label{tab:app-pass2}
\centering
\normalsize
\setlength{\tabcolsep}{0.25pt}
\renewcommand{\arraystretch}{0.96}
\setlength{\aboverulesep}{0.3ex}
\setlength{\belowrulesep}{0.35ex}
\begin{tabular}{p{\dimexpr0.14\textwidth-2\tabcolsep\relax}>{\raggedright\arraybackslash}p{\dimexpr0.20\textwidth-2\tabcolsep\relax}*{5}{>{\centering\arraybackslash}p{\dimexpr0.132\textwidth-2\tabcolsep\relax}}}
\toprule
 & & \multicolumn{5}{c}{\textbf{Math Reasoning} (pass@2)} \\
\cmidrule(lr){3-7}
\textbf{Model} & \textbf{Method} & AMC23 & AIME24 & AIME25 & HMMT25 & Avg. \\
\midrule
\multirow{7}{=}{\centering Qwen3-0.6B\\(Thinking)} & Base & $56.41{\pm}0.97$ & $14.05{\pm}1.10$ & $19.74{\pm}1.63$ & $11.48{\pm}0.17$ & $25.42{\pm}0.35$ \\
 & OPSD & $56.63{\pm}1.46$ & $16.39{\pm}1.23$ & $19.95{\pm}1.05$ & $\mathbf{12.79}{\pm}0.80$ & $26.44{\pm}0.76$ \\
 & SDPO & $56.11{\pm}0.10$ & $14.44{\pm}0.80$ & $20.19{\pm}1.88$ & $10.79{\pm}2.17$ & $25.38{\pm}0.70$ \\
 & GRPO & $56.75{\pm}0.96$ & $14.69{\pm}0.93$ & $\mathbf{22.94}{\pm}0.16$ & $11.56{\pm}1.07$ & $26.49{\pm}0.22$ \\
 & RLSD & $\mathbf{57.57}{\pm}2.43$ & $\underline{16.79}{\pm}0.73$ & $21.13{\pm}0.36$ & $10.72{\pm}0.40$ & $\underline{26.55}{\pm}0.39$ \\
 & \mbox{OPSD + EOPD} & $56.52{\pm}0.56$ & $14.85{\pm}2.14$ & $21.94{\pm}0.24$ & $11.03{\pm}1.75$ & $26.08{\pm}0.31$ \\
\rowcolor[gray]{0.85}  & \textbf{\mbox{\projectname{} (Ours)}} & $\underline{56.81}{\pm}0.77$ & $\mathbf{17.17}{\pm}0.82$ & $\underline{22.28}{\pm}0.54$ & $\underline{12.20}{\pm}1.39$ & $\mathbf{27.11}{\pm}0.37$ \\
\midrule
\multirow{7}{=}{\centering Qwen3-1.7B\\(Thinking)} & Base & $82.44{\pm}0.36$ & $59.14{\pm}1.90$ & $41.86{\pm}0.19$ & $25.96{\pm}1.24$ & $52.35{\pm}0.24$ \\
 & OPSD & $82.34{\pm}0.48$ & $58.31{\pm}1.94$ & $\underline{43.56}{\pm}1.89$ & $\underline{27.72}{\pm}1.12$ & $52.99{\pm}0.79$ \\
 & SDPO & $82.85{\pm}0.72$ & $59.87{\pm}2.21$ & $42.85{\pm}2.48$ & $26.75{\pm}2.21$ & $\underline{53.08}{\pm}0.70$ \\
 & GRPO & $81.76{\pm}0.23$ & $59.20{\pm}1.02$ & $42.74{\pm}1.09$ & $26.76{\pm}1.39$ & $52.62{\pm}0.57$ \\
 & RLSD & $\underline{83.14}{\pm}0.39$ & $\underline{61.35}{\pm}0.21$ & $41.59{\pm}2.17$ & $26.10{\pm}1.01$ & $53.05{\pm}0.66$ \\
 & \mbox{OPSD + EOPD} & $81.14{\pm}0.85$ & $60.64{\pm}1.71$ & $41.79{\pm}1.00$ & $25.77{\pm}1.21$ & $52.33{\pm}0.14$ \\
\rowcolor[gray]{0.85}  & \textbf{\mbox{\projectname{} (Ours)}} & $\mathbf{86.10}{\pm}0.28$ & $\mathbf{62.86}{\pm}0.63$ & $\mathbf{49.32}{\pm}0.98$ & $\mathbf{29.95}{\pm}1.08$ & $\mathbf{57.06}{\pm}0.52$ \\
\midrule
\multirow{7}{=}{\centering Qwen3-4B-\\Instruct-2507} & Base & $89.44{\pm}0.20$ & $71.11{\pm}0.17$ & $56.45{\pm}2.38$ & $33.00{\pm}0.93$ & $62.50{\pm}0.82$ \\
 & OPSD & $89.14{\pm}0.41$ & $\underline{71.94}{\pm}1.83$ & $56.09{\pm}1.29$ & $\mathbf{34.26}{\pm}1.42$ & $\underline{62.86}{\pm}0.68$ \\
 & SDPO & $88.79{\pm}0.19$ & $71.28{\pm}0.33$ & $55.66{\pm}2.04$ & $31.97{\pm}1.21$ & $61.92{\pm}0.71$ \\
 & GRPO & $89.26{\pm}0.55$ & $71.44{\pm}0.50$ & $54.86{\pm}2.06$ & $33.25{\pm}2.14$ & $62.20{\pm}0.19$ \\
 & RLSD & $89.54{\pm}0.83$ & $\mathbf{72.16}{\pm}0.19$ & $\underline{56.64}{\pm}0.95$ & $32.80{\pm}0.21$ & $62.78{\pm}0.17$ \\
 & \mbox{OPSD + EOPD} & $\underline{89.60}{\pm}0.40$ & $71.25{\pm}2.79$ & $55.30{\pm}1.00$ & $32.44{\pm}0.48$ & $62.14{\pm}0.64$ \\
\rowcolor[gray]{0.85}  & \textbf{\mbox{\projectname{} (Ours)}} & $\mathbf{90.92}{\pm}0.39$ & $70.65{\pm}0.76$ & $\mathbf{60.89}{\pm}1.48$ & $\underline{34.25}{\pm}0.65$ & $\mathbf{64.18}{\pm}0.33$ \\
\midrule
\multirow{7}{=}{\centering Olmo-3-7B-\\Instruct} & Base & $87.40{\pm}0.30$ & $63.84{\pm}0.65$ & $47.07{\pm}0.73$ & $28.08{\pm}0.92$ & $56.60{\pm}0.50$ \\
 & OPSD & $87.37{\pm}0.41$ & $\underline{63.90}{\pm}1.77$ & $46.72{\pm}1.64$ & $26.38{\pm}0.31$ & $56.09{\pm}0.93$ \\
 & SDPO & $87.38{\pm}0.13$ & $62.15{\pm}1.27$ & $\mathbf{49.31}{\pm}1.54$ & $\mathbf{29.88}{\pm}0.94$ & $\underline{57.18}{\pm}0.65$ \\
 & GRPO & $\mathbf{88.14}{\pm}0.74$ & $62.26{\pm}0.35$ & $46.99{\pm}0.76$ & $28.76{\pm}1.35$ & $56.54{\pm}0.22$ \\
 & RLSD & $87.53{\pm}0.45$ & $63.67{\pm}1.76$ & $48.24{\pm}0.81$ & $\underline{29.07}{\pm}0.34$ & $57.13{\pm}0.40$ \\
 & \mbox{OPSD + EOPD} & $87.00{\pm}0.28$ & $63.10{\pm}0.77$ & $48.56{\pm}2.67$ & $28.94{\pm}1.52$ & $56.90{\pm}0.54$ \\
\rowcolor[gray]{0.85}  & \textbf{\mbox{\projectname{} (Ours)}} & $\underline{87.92}{\pm}0.81$ & $\mathbf{67.39}{\pm}1.35$ & $\underline{48.86}{\pm}1.97$ & $28.86{\pm}0.86$ & $\mathbf{58.26}{\pm}0.59$ \\
\bottomrule
\end{tabular}
\end{table}

\begin{table}[H]
\caption{Mathematical reasoning unbiased pass@4 (\%).}
\label{tab:app-pass4}
\centering
\normalsize
\setlength{\tabcolsep}{0.25pt}
\renewcommand{\arraystretch}{0.96}
\setlength{\aboverulesep}{0.3ex}
\setlength{\belowrulesep}{0.35ex}
\begin{tabular}{p{\dimexpr0.14\textwidth-2\tabcolsep\relax}>{\raggedright\arraybackslash}p{\dimexpr0.20\textwidth-2\tabcolsep\relax}*{5}{>{\centering\arraybackslash}p{\dimexpr0.132\textwidth-2\tabcolsep\relax}}}
\toprule
 & & \multicolumn{5}{c}{\textbf{Math Reasoning} (pass@4)} \\
\cmidrule(lr){3-7}
\textbf{Model} & \textbf{Method} & AMC23 & AIME24 & AIME25 & HMMT25 & Avg. \\
\midrule
\multirow{7}{=}{\centering Qwen3-0.6B\\(Thinking)} & Base & $66.30{\pm}0.67$ & $19.70{\pm}2.15$ & $24.67{\pm}1.98$ & $15.02{\pm}0.31$ & $31.42{\pm}0.71$ \\
 & OPSD & $66.28{\pm}1.94$ & $23.39{\pm}2.33$ & $25.01{\pm}1.61$ & $\mathbf{16.84}{\pm}1.19$ & $32.88{\pm}1.12$ \\
 & SDPO & $65.01{\pm}0.57$ & $21.34{\pm}1.46$ & $25.35{\pm}2.85$ & $14.25{\pm}2.30$ & $31.49{\pm}0.91$ \\
 & GRPO & $\underline{66.38}{\pm}0.94$ & $20.94{\pm}2.31$ & $\mathbf{29.06}{\pm}0.48$ & $14.94{\pm}1.58$ & $32.83{\pm}0.58$ \\
 & RLSD & $\mathbf{66.82}{\pm}2.26$ & $\underline{23.79}{\pm}1.18$ & $27.00{\pm}0.90$ & $14.42{\pm}0.77$ & $\underline{33.01}{\pm}0.01$ \\
 & \mbox{OPSD + EOPD} & $65.42{\pm}0.67$ & $21.59{\pm}3.90$ & $27.59{\pm}0.92$ & $13.86{\pm}2.47$ & $32.11{\pm}0.75$ \\
\rowcolor[gray]{0.85}  & \textbf{\mbox{\projectname{} (Ours)}} & $65.85{\pm}0.28$ & $\mathbf{25.02}{\pm}1.33$ & $\underline{27.65}{\pm}0.67$ & $\underline{16.09}{\pm}1.30$ & $\mathbf{33.65}{\pm}0.16$ \\
\midrule
\multirow{7}{=}{\centering Qwen3-1.7B\\(Thinking)} & Base & $87.14{\pm}0.43$ & $68.65{\pm}2.28$ & $48.73{\pm}0.64$ & $31.31{\pm}2.08$ & $58.96{\pm}0.75$ \\
 & OPSD & $87.65{\pm}0.31$ & $67.16{\pm}2.89$ & $\underline{51.35}{\pm}2.24$ & $\underline{33.75}{\pm}1.32$ & $\underline{59.97}{\pm}1.17$ \\
 & SDPO & $87.72{\pm}0.87$ & $69.17{\pm}1.62$ & $49.76{\pm}2.75$ & $32.14{\pm}2.90$ & $59.69{\pm}0.87$ \\
 & GRPO & $86.83{\pm}0.35$ & $69.06{\pm}1.03$ & $49.83{\pm}0.49$ & $33.23{\pm}2.27$ & $59.74{\pm}0.75$ \\
 & RLSD & $\underline{88.02}{\pm}0.72$ & $\mathbf{70.87}{\pm}0.65$ & $49.51{\pm}1.84$ & $31.29{\pm}1.54$ & $59.92{\pm}0.48$ \\
 & \mbox{OPSD + EOPD} & $85.68{\pm}1.81$ & $\mathbf{70.87}{\pm}2.23$ & $49.64{\pm}1.00$ & $31.13{\pm}2.25$ & $59.33{\pm}0.40$ \\
\rowcolor[gray]{0.85}  & \textbf{\mbox{\projectname{} (Ours)}} & $\mathbf{90.91}{\pm}0.73$ & $\underline{70.54}{\pm}1.37$ & $\mathbf{57.71}{\pm}1.98$ & $\mathbf{34.83}{\pm}1.26$ & $\mathbf{63.50}{\pm}0.99$ \\
\midrule
\multirow{7}{=}{\centering Qwen3-4B-\\Instruct-2507} & Base & $91.67{\pm}0.12$ & $77.81{\pm}0.24$ & $64.22{\pm}2.16$ & $36.89{\pm}1.13$ & $67.65{\pm}0.64$ \\
 & OPSD & $91.36{\pm}0.77$ & $78.85{\pm}1.39$ & $64.51{\pm}1.05$ & $\mathbf{38.20}{\pm}2.27$ & $68.23{\pm}0.72$ \\
 & SDPO & $91.75{\pm}0.48$ & $78.45{\pm}1.12$ & $62.88{\pm}2.11$ & $36.11{\pm}1.10$ & $67.30{\pm}0.71$ \\
 & GRPO & $91.93{\pm}0.56$ & $\underline{79.07}{\pm}0.99$ & $61.84{\pm}2.34$ & $37.09{\pm}3.13$ & $67.48{\pm}0.45$ \\
 & RLSD & $\underline{91.99}{\pm}0.94$ & $\mathbf{79.73}{\pm}0.65$ & $\underline{64.67}{\pm}1.17$ & $36.79{\pm}0.69$ & $\underline{68.30}{\pm}0.27$ \\
 & \mbox{OPSD + EOPD} & $91.93{\pm}0.33$ & $77.64{\pm}3.05$ & $61.66{\pm}1.89$ & $36.41{\pm}1.15$ & $66.91{\pm}0.34$ \\
\rowcolor[gray]{0.85}  & \textbf{\mbox{\projectname{} (Ours)}} & $\mathbf{93.00}{\pm}0.52$ & $76.29{\pm}0.32$ & $\mathbf{67.94}{\pm}1.38$ & $\underline{38.04}{\pm}0.97$ & $\mathbf{68.82}{\pm}0.36$ \\
\midrule
\multirow{7}{=}{\centering Olmo-3-7B-\\Instruct} & Base & $90.69{\pm}0.33$ & $73.01{\pm}1.02$ & $55.42{\pm}1.90$ & $32.72{\pm}1.30$ & $62.96{\pm}0.74$ \\
 & OPSD & $90.76{\pm}0.72$ & $\underline{73.33}{\pm}1.87$ & $54.37{\pm}1.36$ & $29.51{\pm}0.54$ & $61.99{\pm}1.07$ \\
 & SDPO & $90.77{\pm}0.25$ & $72.20{\pm}1.28$ & $\mathbf{57.58}{\pm}1.42$ & $\mathbf{34.62}{\pm}1.90$ & $\underline{63.79}{\pm}0.86$ \\
 & GRPO & $\mathbf{91.90}{\pm}0.52$ & $70.90{\pm}0.98$ & $53.90{\pm}0.35$ & $33.33{\pm}2.45$ & $62.51{\pm}0.36$ \\
 & RLSD & $91.04{\pm}0.54$ & $72.45{\pm}1.66$ & $\underline{56.75}{\pm}0.87$ & $33.39{\pm}1.12$ & $63.41{\pm}0.63$ \\
 & \mbox{OPSD + EOPD} & $91.02{\pm}0.13$ & $71.88{\pm}0.18$ & $56.71{\pm}2.54$ & $\underline{33.55}{\pm}0.92$ & $63.29{\pm}0.63$ \\
\rowcolor[gray]{0.85}  & \textbf{\mbox{\projectname{} (Ours)}} & $\underline{91.25}{\pm}0.66$ & $\mathbf{76.34}{\pm}1.44$ & $56.06{\pm}1.78$ & $33.49{\pm}1.73$ & $\mathbf{64.28}{\pm}0.74$ \\
\bottomrule
\end{tabular}
\end{table}

\FloatBarrier

\subsection{Out-of-domain Evaluation}
\label{app:out-of-domain-results}

\begin{table}[H]
\caption{Out-of-domain evaluation on logical reasoning and code generation. ZebraLogic uses whole-puzzle correctness; best means are bold within each model and dataset.}
\label{tab:app-out-of-domain}
\centering
\normalsize
\setlength{\tabcolsep}{2pt}
\renewcommand{\arraystretch}{1.05}
\setlength{\aboverulesep}{0.3ex}
\setlength{\belowrulesep}{0.3ex}
\begin{tabular}{p{\dimexpr0.23\textwidth-2\tabcolsep\relax}>{\raggedright\arraybackslash}p{\dimexpr0.21\textwidth-2\tabcolsep\relax}*{4}{>{\centering\arraybackslash}p{\dimexpr0.14\textwidth-2\tabcolsep\relax}}}
\toprule
\textbf{Model} & \textbf{Method} & mean@16 & pass@2 & pass@4 & pass@8 \\
\midrule
\multicolumn{6}{c}{\textbf{Logic: ZebraLogic}} \\
\midrule
 & Base & $28.83{\pm}0.30$ & $38.93{\pm}0.35$ & $48.54{\pm}0.23$ & $56.47{\pm}0.63$ \\
\rowcolor[gray]{0.85} \cellcolor{white}\multirow{-2}{=}{Qwen3-0.6B (Thinking)} & \textbf{\mbox{\projectname{} (Ours)}} & $\mathbf{28.96}{\pm}0.21$ & $\mathbf{39.27}{\pm}0.32$ & $\mathbf{48.86}{\pm}0.53$ & $\mathbf{56.54}{\pm}0.71$ \\
\midrule
 & Base & $62.54{\pm}1.05$ & $69.14{\pm}0.99$ & $74.19{\pm}1.21$ & $78.26{\pm}1.32$ \\
\rowcolor[gray]{0.85} \cellcolor{white}\multirow{-2}{=}{Qwen3-1.7B (Thinking)} & \textbf{\mbox{\projectname{} (Ours)}} & $\mathbf{67.35}{\pm}0.43$ & $\mathbf{73.57}{\pm}0.62$ & $\mathbf{78.75}{\pm}0.84$ & $\mathbf{83.77}{\pm}0.98$ \\
\midrule
 & Base & $79.15{\pm}0.16$ & $85.65{\pm}0.23$ & $89.61{\pm}0.41$ & $92.35{\pm}0.64$ \\
\rowcolor[gray]{0.85} \cellcolor{white}\multirow{-2}{=}{Qwen3-4B-Instruct-2507} & \textbf{\mbox{\projectname{} (Ours)}} & $\mathbf{82.23}{\pm}0.21$ & $\mathbf{87.57}{\pm}0.10$ & $\mathbf{90.65}{\pm}0.35$ & $\mathbf{92.83}{\pm}0.82$ \\
\midrule
 & Base & $32.69{\pm}0.39$ & $42.38{\pm}0.37$ & $50.61{\pm}0.25$ & $56.56{\pm}0.05$ \\
\rowcolor[gray]{0.85} \cellcolor{white}\multirow{-2}{=}{Olmo-3-7B-Instruct} & \textbf{\mbox{\projectname{} (Ours)}} & $\mathbf{33.90}{\pm}0.43$ & $\mathbf{43.96}{\pm}0.53$ & $\mathbf{52.16}{\pm}0.71$ & $\mathbf{57.94}{\pm}0.95$ \\
\midrule
\multicolumn{6}{c}{\textbf{Coding: HumanEval}} \\
\midrule
 & Base & $52.55{\pm}0.31$ & $63.96{\pm}0.45$ & $72.60{\pm}0.34$ & $78.42{\pm}0.48$ \\
\rowcolor[gray]{0.85} \cellcolor{white}\multirow{-2}{=}{Qwen3-0.6B (Thinking)} & \textbf{\mbox{\projectname{} (Ours)}} & $\mathbf{53.09}{\pm}0.70$ & $\mathbf{64.26}{\pm}0.76$ & $\mathbf{72.84}{\pm}0.71$ & $\mathbf{79.01}{\pm}0.72$ \\
\midrule
 & Base & $83.75{\pm}0.21$ & $88.91{\pm}0.22$ & $91.55{\pm}0.17$ & $93.17{\pm}0.48$ \\
\rowcolor[gray]{0.85} \cellcolor{white}\multirow{-2}{=}{Qwen3-1.7B (Thinking)} & \textbf{\mbox{\projectname{} (Ours)}} & $\mathbf{84.50}{\pm}0.13$ & $\mathbf{89.12}{\pm}0.19$ & $\mathbf{91.79}{\pm}0.34$ & $\mathbf{93.77}{\pm}0.34$ \\
\midrule
 & Base & $83.96{\pm}0.16$ & $86.41{\pm}0.24$ & $88.37{\pm}0.36$ & $89.90{\pm}0.51$ \\
\rowcolor[gray]{0.85} \cellcolor{white}\multirow{-2}{=}{Qwen3-4B-Instruct-2507} & \textbf{\mbox{\projectname{} (Ours)}} & $\mathbf{84.76}{\pm}0.22$ & $\mathbf{87.13}{\pm}0.03$ & $\mathbf{88.92}{\pm}0.12$ & $\mathbf{90.67}{\pm}0.20$ \\
\midrule
 & Base & $78.18{\pm}0.51$ & $85.33{\pm}0.47$ & $89.68{\pm}0.44$ & $92.05{\pm}0.25$ \\
\rowcolor[gray]{0.85} \cellcolor{white}\multirow{-2}{=}{Olmo-3-7B-Instruct} & \textbf{\mbox{\projectname{} (Ours)}} & $\mathbf{78.66}{\pm}0.19$ & $\mathbf{85.78}{\pm}0.34$ & $\mathbf{90.14}{\pm}0.52$ & $\mathbf{92.52}{\pm}0.74$ \\
\bottomrule
\end{tabular}
\end{table}

\FloatBarrier

\subsection{Ablation Results}
\label{app:ablation-results}
\begin{table}[H]
\caption{Per-benchmark mean@16 with standard deviations for the component ablation on Qwen3-1.7B
Thinking (Figure~\ref{fig:ablation}(a)). Avg. equally weights the four benchmarks. Best means are bold and second-best distinct means underlined. AGD denotes adaptive gap damping.}
\label{tab:app-ablation-mean16}
\centering
\normalsize
\setlength{\tabcolsep}{4pt}
\begin{tabular}{lccccc}
\toprule
Method & AMC23 & AIME24 & AIME25 & HMMT25 & Avg. \\
\midrule
Base & $75.50$\,{$\pm0.28$} & $46.04$\,{$\pm1.78$} & $34.58$\,{$\pm0.29$} & $21.88$\,{$\pm0.78$} & $44.50$\,{$\pm0.27$} \\
OPSD & $74.87$\,{$\pm0.46$} & $45.97$\,{$\pm1.58$} & $35.00$\,{$\pm1.36$} & $22.71$\,{$\pm0.61$} & $44.64$\,{$\pm0.41$} \\
\;+\,EGT & $76.46$\,{$\pm0.99$} & $48.26$\,{$\pm1.21$} & $36.81$\,{$\pm1.49$} & $22.57$\,{$\pm0.20$} & $46.02$\,{$\pm0.40$} \\
\;+\,BER & $77.26$\,{$\pm0.38$} & $52.01$\,{$\pm1.13$} & $37.15$\,{$\pm0.94$} & $24.03$\,{$\pm0.35$} & $47.61$\,{$\pm0.64$} \\
\;+\,BER\,+\,AGD & $\underline{77.66}$\,{$\pm0.68$} & $\underline{52.15}$\,{$\pm1.58$} & $\mathbf{39.44}$\,{$\pm0.69$} & $\underline{24.31}$\,{$\pm1.83$} & $\underline{48.39}$\,{$\pm0.55$} \\
\rowcolor[gray]{0.85} \textbf{\projectname{} (Ours)} & $\mathbf{79.02}$\,{$\pm0.41$} & $\mathbf{52.29}$\,{$\pm0.45$} & $\underline{39.24}$\,{$\pm0.26$} & $\mathbf{25.00}$\,{$\pm0.90$} & $\mathbf{48.89}$\,{$\pm0.32$} \\
\bottomrule
\end{tabular}
\end{table}

\FloatBarrier

\clearpage
\section{Mechanism Probe Settings and Results}
\label{app:mechanism-probes}

\subsection{Rollout and Probe Settings}
\label{app:probe-settings}
Figures~\ref{fig:ablation}(b,c) use the same diagnostic batch of student
responses. We sample 256 distinct problems from the DeepMath level 5--7
training pool with seed 42 and generate eight responses per problem.
All three conditions use the same frozen Qwen3-1.7B base weights in Thinking
mode, with no parameter updates. The student receives only the original
problem; the reference-conditioned teacher additionally receives its reference
solution and answer; the EGT teacher instead receives one retrieved solved
neighbor. We retain the same problems and exemplar assignments when generating
the teachers' own responses.

\begin{table}[H]
\caption{Diagnostic rollout settings. Indices $i$ and $g$ start at zero.
The two teacher conditions share sampling seeds. Teacher generation budgets
are also limited by the remaining context window.}
\label{tab:probe-settings}
\centering
\normalsize
\setlength{\tabcolsep}{7pt}
\renewcommand{\arraystretch}{1.08}
\begin{tabular}{ll}
\toprule
Parameter & Value \\
\midrule
Problems / responses per problem & 256 / 8 \\
Sampling temperature & 1.0 \\
Top-$p$ / top-$k$ & 0.95 / 20 \\
Maximum response length & 16,384 tokens \\
Maximum total context length & 40,960 tokens \\
Student prompt length limit & 2,048 tokens \\
Student sampling seed & $42+i$ for problem index $i$ \\
Teacher sampling seed & $42+8i+g$ for response slot $g$ \\
Entropy and KL scoring temperature & 1.0 \\
\bottomrule
\end{tabular}
\end{table}

\paragraph{Teacher-context probe.}
For Figure~\ref{fig:ablation}(b), we score each teacher both on its own generated
response and on the corresponding student response. Entropy and
teacher-to-student forward KL are computed over the full vocabulary, without
top-$k$ truncation or clipping. We average over response tokens within each
trajectory, then equally over 2,047 matched response slots; one slot is
excluded because the EGT prompt and student response exceed the context window.
Response lengths use all 2,048 generations per condition and include generated
end-of-sequence tokens. Relative entropy increase is the difference between
the two mean entropies divided by the own-prefix mean, rather than an average
of per-response ratios.

\paragraph{AGD gradient probe.}
For Figure~\ref{fig:ablation}(c), we reuse all 2,048 student responses with the
reference-conditioned teacher. At each response position, we compute the
$\ell_2$ norm of the routed loss gradient with respect to the full student
logit vector, using teacher top-100 support plus a tail bucket and a
per-element cap of 0.05. We compare the same fixed distributions with AGD
off and on, stopping gradients through the routing decision and damping weight.
The gap scale is computed over valid tokens in fixed batches of 32 responses,
ordered by problem and response slot. Gradient shares are normalized within
each response, averaged over responses within a problem, and then over problems.
High entropy means $H\geq\ln 2$; the first and second letters in HH, HL, LH,
and LL refer to student and teacher entropy, respectively. This is an offline
logit-gradient probe, not a replay of parameter updates during training.

\subsection{Detailed Probe Results}
\label{app:probe-results}
Table~\ref{tab:teacher-context-probe} reports the teacher-context comparison.
EGT reduces forward KL on the same student
prefixes by 25.9\%. Own-prefix entropy and response length describe each
condition's separate trajectories; the forward-KL comparison uses shared
student prefixes.

\begin{table}[H]
\caption{Teacher-context probe underlying Figure~\ref{fig:ablation}(b).
Lengths use 2,048 responses per condition; entropy and KL
use 2,047 matched slots. Student zeros denote identical-prefix or self comparisons.}
\label{tab:teacher-context-probe}
\centering
\normalsize
\setlength{\tabcolsep}{5pt}
\renewcommand{\arraystretch}{1.08}
\begin{tabular}{lrrr}
\toprule
Metric & Student & Reference teacher & EGT teacher \\
\midrule
Mean response length (tokens) & 9,139.7 & 835.6 & 2,089.5 \\
Own-prefix entropy (nats) & 0.2679 & 0.1900 & 0.2605 \\
Forward KL to student (nats) & 0 & 0.1549 & 0.1148 \\
\bottomrule
\end{tabular}
\end{table}

Table~\ref{tab:agd-gradient-probe} reports all four entropy buckets for
Figure~\ref{fig:ablation}(c). HL and LH together contain 1,436,048 of
18,718,121 response tokens (7.67\%). Their gradient shares increase by
3.43 and 4.89 percentage points, respectively. 

\begin{table}[H]
\caption{AGD gradient probe on 2,048 fixed student responses.
Token shares pool all response tokens; gradient shares use the response- and
problem-level normalization described in Appendix~\ref{app:probe-settings}.
Changes are computed from the displayed gradient shares.}
\label{tab:agd-gradient-probe}
\centering
\normalsize
\setlength{\tabcolsep}{7pt}
\renewcommand{\arraystretch}{1.08}
\begin{tabular}{lrrrr}
\toprule
Bucket & Token share (\%) & AGD off (\%) & AGD on (\%) & Change (pp) \\
\midrule
HH & 14.40 & 45.57 & 43.26 & $-2.31$ \\
HL & 2.33 & 8.99 & 12.42 & $+3.43$ \\
LH & 5.34 & 11.19 & 16.08 & $+4.89$ \\
LL & 77.93 & 34.25 & 28.24 & $-6.01$ \\
\bottomrule
\end{tabular}
\end{table}

\paragraph{Distributional agreement at strongly damped positions.}
\label{app:agd-overlap}
Using the same reference-teacher gradient probe, we compute full-vocabulary
probability-mass overlap $O_t=\sum_v\min\{p_t(v),q_t(v)\}$ on student
prefixes. Table~\ref{tab:agd-overlap} reports token-pooled statistics
for strongly damped positions ($c_t\leq0.1$).

\begin{table}[H]
\caption{Overlap at strongly damped positions. Percentages are computed within each row.}
\label{tab:agd-overlap}
\centering
\normalsize
\setlength{\tabcolsep}{6pt}
\renewcommand{\arraystretch}{1.08}
\begin{tabular}{lrrr}
\toprule
Subset & Positions & Mean overlap & $O_t<0.5$ (\%) \\
\midrule
All strongly damped positions & 12,708,962 & 0.9848 & 0.380 \\
\quad with $H(p_t)\geq\ln 2$ & 766,682 & 0.8343 & 3.662 \\
\bottomrule
\end{tabular}
\end{table}

Positions with both $c_t\leq0.1$ and $O_t<0.5$ account for
0.464\% of $\sum_t g_t$ pooled over all response tokens, where
$g_t=\|\nabla_{z_t}D_t^{\mathrm{BER}}\|_2$ is the logit-gradient norm
before damping.
\FloatBarrier

\clearpage
\section{Entropy Gap on a Shared Candidate Set}
\label{app:entropy-gap-theory}

This appendix gives the precise conditions and proof for the proposition
motivating the entropy-gap routing signal in Section~\ref{sec:rein-ead}.
At a fixed response position, write $p=p_t$ and $q=q_t$ for the student
and teacher distributions over a finite vocabulary
$\mathcal V$, with entropy $H(r)=-\sum_{v\in\mathcal V}r(v)\log r(v)$.
All logarithms are natural, with $0\log 0=0$.

\paragraph{Proposition (Entropy-Gap Stability).}
Let $S\subseteq\mathcal V$ be a nonempty shared candidate set satisfying
$p(S),q(S)\geq 1-\eta$ for $0\leq\eta\leq 1/2$.
For $r\in\{p,q\}$, define its restriction to $S$ by
$r^S(v)=r(v)/r(S)$ for $v\in S$. Write
\[
\Delta=H(p)-H(q),
\qquad
\Delta^S=H(p^S)-H(q^S).
\]
Then
\[
|\Delta-\Delta^S|\leq B_\eta,
\qquad
B_\eta=2h(\eta)+2\eta\log|\mathcal V|,
\]
where $h(a)=-a\log a-(1-a)\log(1-a)$ is binary entropy.
Consequently, if $|\Delta|>B_\eta$, then
$\operatorname{sign}(\Delta)=\operatorname{sign}(\Delta^S)$.

\begin{proof}
For either distribution $r$, let $\alpha_r=r(S^c)\leq\eta$.
Decomposing entropy according to membership in $S$ gives
\[
H(r)=h(\alpha_r)+(1-\alpha_r)H(r^S)
     +\alpha_r H(r^{S^c}),
\]
where $r^{S^c}$ is the normalized restriction to the complement; the last
term is defined as zero when $\alpha_r=0$.
Since both conditional entropies lie between $0$ and $\log|\mathcal V|$,
and $h$ is increasing on $[0,1/2]$,
\[
|H(r)-H(r^S)|
\leq h(\alpha_r)+\alpha_r\log|\mathcal V|
\leq h(\eta)+\eta\log|\mathcal V|.
\]
Applying this bound to $p$ and $q$ and using the triangle inequality yields
$|\Delta-\Delta^S|\leq B_\eta$.
If $|\Delta|>B_\eta$, this perturbation cannot change its sign.
\end{proof}

\paragraph{Interpretation.}
The quantity $\exp(H(r^S))$ measures the effective number of alternatives
under $r$ within $S$. Thus, when $|\Delta|>B_\eta$, a negative entropy gap
indicates fewer effective alternatives under $p$ than under $q$ within
the shared set, while a positive gap indicates more. For a fixed finite
vocabulary, the error bound vanishes as $\eta\to 0$.

\paragraph{Empirical check.}
Using the initial Qwen3-1.7B Thinking model, we randomly select 64 DeepMath
level 5--7 problems, one saved student response per problem, and 128 token positions per response
(8,192 positions in total). The frozen EGT teacher scores the same student
prefixes. At each position, $S$ is the union of the student's and teacher's
top-$K$ tokens, and $\eta=\max\{p(S^c),q(S^c)\}$.
This set is used only for analysis; the training loss uses teacher
top-$K_{\mathrm v}$ tokens plus a tail bucket.
Probability masses and entropies use FP64 arithmetic on BF16 model logits.

\begin{table}[H]
\caption{Shared-candidate diagnostics (\% of the 8,192 sampled positions).
Mass coverage requires $p(S),q(S)\geq99.9\%$; condition coverage requires
$\eta\leq1/2$ and $|\Delta|>B_\eta$; sign agreement compares $\Delta$ and $\Delta^S$.}
\label{tab:entropy-gap-conditions}
\centering
\normalsize
\setlength{\tabcolsep}{10pt}
\renewcommand{\arraystretch}{1.08}
\begin{tabular}{lrrr}
\toprule
Top-$K$ & Mass coverage & Condition coverage & Sign agreement \\
\midrule
50 & 98.14 & 98.61 & 99.90 \\
100 (main) & 99.19 & 99.13 & 99.94 \\
200 & 99.65 & 99.39 & 99.95 \\
\bottomrule
\end{tabular}
\end{table}

Table~\ref{tab:entropy-gap-conditions} shows that the sufficient conditions
hold at over 98\% of positions across all three candidate-set sizes.
At $K=100$, all five sign disagreements occur outside the 8,121 positions
satisfying these conditions. These results support the proposition's
applicability in this initial-model, fixed-prefix probe.

\end{document}